\documentclass[11pt]{article}

\usepackage[preprint]{acl}

\usepackage{times}
\usepackage{latexsym}
\usepackage[T1]{fontenc}
\usepackage[utf8]{inputenc}
\usepackage{microtype}
\usepackage{inconsolata}
\usepackage{graphicx}
\usepackage{booktabs}
\usepackage{amsmath,amssymb,amsfonts,amsthm}
\usepackage{mathtools}
\usepackage{bm}
\usepackage{bbm}
\usepackage{multirow}
\usepackage{xcolor}
\usepackage{colortbl}
\usepackage{array}
\usepackage{tabularx}
\usepackage{enumitem}
\usepackage[ruled,vlined]{algorithm2e}
\SetKwInput{KwInput}{Input}
\SetKwInput{KwOutput}{Output}

\newtheorem{proposition}{Proposition}

\definecolor{tabhead}{RGB}{36,52,71}
\definecolor{simred}{RGB}{165,30,45}
\newcommand{\Th}[1]{\textcolor{white}{\textbf{#1}}}
\newcommand{\ourname}[1]{\textcolor{tabhead}{\textbf{#1}}}
\newcommand{\tabend}{\arrayrulecolor{tabhead}\hline\arrayrulecolor{black}}

\newcommand{\method}{\textsc{Squelch}}
\newcommand{\E}{\mathbb{E}}
\newcommand{\R}{\mathbb{R}}
\newcommand{\soft}{\mathcal{S}}
\newcommand{\hA}{\hat{A}}
\newcommand{\hs}{\hat\sigma}
\newcommand{\supp}{\mathrm{supp}}
\newcommand{\SURE}{\mathrm{SURE}}
\newcommand{\cu}{c_{\mathrm{U}}}

\newcommand{\oursCsL}{87.04}

\newcommand{\bestNameCsL}{FlyLoRA}
\newcommand{\gainBestCsL}{0.93}
\newcommand{\gainLoraCsL}{2.64}
\newcommand{\oursCsQ}{89.28}

\newcommand{\gainBestCsQ}{0.77}
\newcommand{\gainLoraCsQ}{2.18}

\newcommand{\bestNameDomL}{FlyLoRA}
\newcommand{\gainBestDomL}{1.19}

\newcommand{\gainBestDomQ}{0.75}

\newcommand{\gainLoraMin}{2.2}
\newcommand{\gainLoraMax}{3.2}
\newcommand{\gainBestMin}{0.8}
\newcommand{\gainBestMax}{1.2}
\newcommand{\rankOurs}{8.6}

\newcommand{\gapOurs}{0.09}
\newcommand{\gapNoEA}{0.27}
\newcommand{\rankNoEA}{6.9}
\newcommand{\dropNoEA}{0.61}
\newcommand{\rankZeroIn}{8.7}
\newcommand{\rankZeroOOD}{21.4}

\newcommand{\mtDropOurs}{0.47}
\newcommand{\mtDropLora}{1.50}
\newcommand{\mtDropFly}{0.85}
\newcommand{\forgetOurs}{1.13}
\newcommand{\forgetLora}{3.87}
\newcommand{\forgetFly}{2.20}
\newcommand{\glueOurs}{89.71}
\newcommand{\glueLora}{88.41}
\newcommand{\glueBest}{89.21}
\newcommand{\trainOverhead}{5}
\newcommand{\inferOverhead}{5}

\title{How Many Ranks Does a Token Need?\\A Noise-Floor Squelch for Rank-One LoRA Experts}

\author{
  \textbf{Mei Okonkwo\textsuperscript{1}},
  \textbf{Owen Castellanos\textsuperscript{1}},
  \textbf{Ingrid Solberg\textsuperscript{2}},
  \textbf{Yusuf Demirel\textsuperscript{1}}
  \\
  \\
  \textsuperscript{1}University of Edinburgh \\
  \textsuperscript{2}Delft University of Technology
}

\begin{document}
\maketitle

\begin{abstract}
Low-rank adaptation (LoRA) applies every rank-one component of its update to every token, and mixture-of-experts variants replace this with a learned router that activates a fixed number of components. We argue that the number of components a token needs is a detection question with a classical answer. A token's responses to the rank-one components of a trained LoRA consist of a few large values on a Gaussian noise floor whose width matches that of an untrained adapter, and the floor carries cross-talk rather than adaptation. We introduce \method{}, a router-free mixture of rank-one experts that (i) soft-thresholds each token's responses at a multiple of its own robust noise level, so the active rank varies per token; (ii) chooses that multiple for each module by Stein's unbiased risk estimate, so no rank budget is specified; and (iii) replaces the resulting marginal keep-or-drop decisions by one explaining-away step on the value Gram matrix, which suppresses redundant co-active components and recruits complementary ones. \method{} adds no trainable parameters to LoRA. We prove that the squelch is within a logarithmic factor of the oracle input-dependent allocation, that the selected threshold is risk-consistent, and that explaining away is a descent step toward the joint sparse code whose error is controlled by the coherence of co-active components. On commonsense reasoning, knowledge, math, code, and GLUE with LLaMA-3.1-8B, Qwen3, and DeBERTaV3, \method{} improves over LoRA by \gainLoraMin--\gainLoraMax{} points and over the strongest mixture baseline by \gainBestMin--\gainBestMax{} points while activating \rankOurs{} of 32 components per token on average.
\end{abstract}

\section{Introduction}
\label{sec:intro}

\begin{figure*}[t]
\centering
\includegraphics[width=\textwidth]{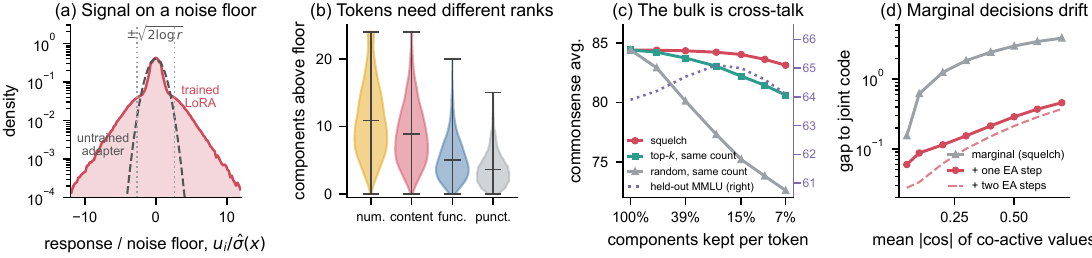}
\caption{\textbf{A token's rank-one responses are a sparse signal on a noise floor} (LLaMA-3.1-8B, LoRA $r{=}32$ trained on commonsense reasoning). (a) Standardized responses of the trained adapter (red) against those of an untrained adapter (dashed): a Gaussian bulk plus heavy tails beyond the universal threshold. (b) The number of above-floor components per token varies with the token type. (c) Zeroing sub-floor responses after training preserves accuracy and improves held-out MMLU (right axis), whereas keeping the same number of components by top-$k$ or at random does not. (d) Gap between the marginal (thresholded) code and the joint sparse code as co-active value vectors become coherent (random dictionaries, exact computation); one explaining-away step closes most of it.}
\label{fig:motivation}
\end{figure*}

Low-rank adaptation (LoRA) fine-tunes a frozen model by adding a product $BA$ of two thin matrices to each linear layer~\citep{hu2022lora,dettmers2023qlora}. Written as a sum $\sum_i b_ia_i^{\!\top}$ of $r$ rank-one components, the update applies all $r$ components to every token. Two lines of work have questioned this uniformity. Rank-allocation methods give different modules different ranks, but a module's rank is still shared by all of its inputs~\citep{zhang2023adalora,ding2023sora,he2025gora,liu2026flexlora}. Mixtures of LoRA experts route tokens among low-rank experts~\citep{li2024mixlora,tian2024hydralora,zhuang2025ldmole}, and recent designs push the expert down to a single rank-one component~\citep{zhang2024trex,zhao2025smora,zou2026flylora,lu2026moram}. How many components fire is still fixed by hand or learned from the task loss: through a top-$k$, a probability mass, a learned threshold network, or a complexity predictor~\citep{zhao2025smora,huang2024harder,liu2024adamole,li2026flexilora}.

We ask the question directly: \emph{how many rank-one components does a token need?} Figure~\ref{fig:motivation} suggests an answer. If we collect a token's responses $u_i=\hat a_i^{\!\top}x$ to the normalized keys of a trained LoRA, a few of them are large and the rest form a Gaussian bulk whose width matches the response of an \emph{untrained} adapter, i.e.\ of random projections (panel a). The number of responses above this floor differs widely across tokens (panel b). The bulk also does little useful work: zeroing it after training preserves in-domain accuracy and even improves held-out retention, whereas removing the same number of components at random, or by a fixed top-$k$, degrades accuracy (panel c). The bulk is cross-talk.

This turns rank allocation into a detection problem, the sparse normal-means problem of signal processing~\citep{donoho1994ideal}, and the problem comes with classical answers. One should threshold relative to a robust estimate of the noise level~\citep{donoho1995denoising,hampel1974influence}. The threshold can be chosen without labels by minimizing Stein's unbiased risk estimate~\citep{stein1981estimation,donoho1995adapting}. When the atoms overlap, marginal thresholding is no longer the right estimator and should give way to the joint sparse code, which iterative shrinkage approximates~\citep{tibshirani1996lasso,daubechies2004iterative}. Panel~(d) shows why the last point matters. As the value vectors of co-active components become coherent, the marginal code drifts far from the joint code, and a single shrinkage step removes most of the drift.

We build \method{} on this view. \method{} keeps LoRA's parameters and changes only how the rank-one components are combined for each token. It (i) \emph{squelches} every response that does not clear a multiple $c_\ell$ of the token's own noise floor, estimated by the median absolute response; (ii) sets $c_\ell$ for each module by minimizing Stein's unbiased risk estimate over the tokens of a batch; and (iii) applies one \emph{explaining-away} step that uses the Gram matrix of the value vectors to turn marginal keep-or-drop decisions into a joint one. The per-token rank is then an outcome rather than a hyperparameter: \method{} uses \rankOurs{} of 32 components per token on average, more on numerals and less on function words, and no component is active for \rankZeroIn\% of token--module pairs. Our contributions are:
\begin{itemize}[leftmargin=*,itemsep=1pt,topsep=2pt]
\item \textbf{A statistical view of per-token rank.} Rank-one responses behave like a sparse signal on a noise floor, so choosing how many components a token uses is sparse detection (\S\ref{sec:intro}, \S\ref{sec:method}).
\item \textbf{A router-free, budget-free method.} \method{} combines noise-floor thresholding, SURE-selected thresholds, and one explaining-away step, and adds no trainable parameters to LoRA (\S\ref{sec:method}).
\item \textbf{Guarantees.} The squelch is within a logarithmic factor of the oracle allocation, the SURE threshold is risk-consistent, and explaining away is a descent step toward the joint code whose error is governed by co-active coherence (\S\ref{sec:theory}).
\item \textbf{Evidence.} On five task families and seven backbones, \method{} outperforms ten LoRA and mixture baselines, with ablations and a controlled study that isolate each step (\S\ref{sec:experiments}, App.~\ref{app:synthetic}).
\end{itemize}

\section{Related Work}
\label{sec:related}

\begin{figure*}[t]
\centering
\includegraphics[width=\textwidth]{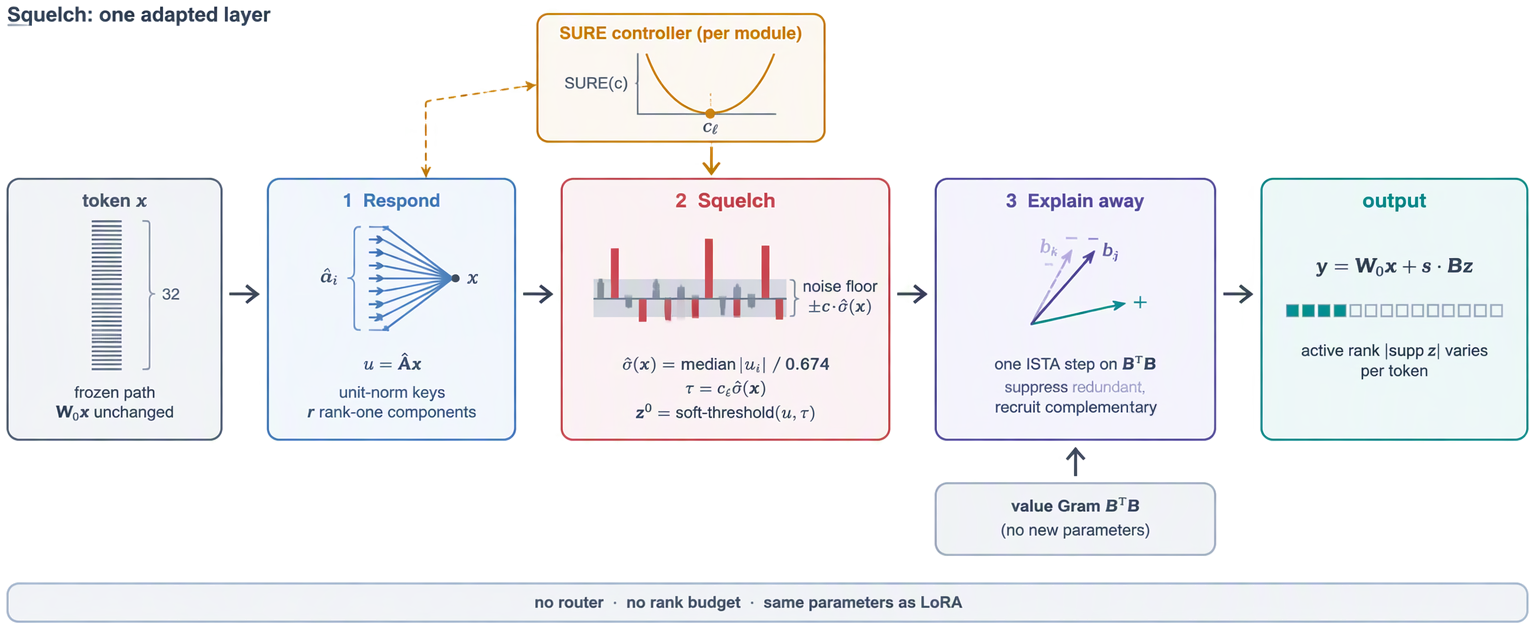}
\caption{\textbf{\method{} in one adapted layer.} The token's responses to $r$ unit-norm keys are squelched at a multiple $c_\ell$ of the token's own noise floor, where $c_\ell$ is chosen per module by Stein's unbiased risk estimate. One explaining-away step on the value Gram $B^{\!\top}\!B$ turns the marginal code into a joint one. The surviving components form the token's update, so its rank varies per token. \method{} uses LoRA's parameters $A,B$ and adds none.}
\label{fig:overview}
\end{figure*}

\paragraph{Rank allocation.}
AdaLoRA prunes singular values by an importance score~\citep{zhang2023adalora}, SoRA gates ranks with a proximal $\ell_1$ step~\citep{ding2023sora}, and later methods allocate ranks by gradients, beam search, learned elasticity, entropy, or variational relevance~\citep{he2025gora,gu2025beamlora,chang2025elalora,liu2026flexlora,doyle2025lrvd}. These methods change \emph{which modules} receive rank; every input of a module shares it. Input-dependent allocation is rarer. RaNA masks ranks and neurons per input for faster inference~\citep{garcia2025rana}, and Flexi-LoRA predicts a rank from input complexity~\citep{li2026flexilora}. Gradient-based importance scores can also mislead~\citep{sawant2026gradient}, which motivates a budget that does not depend on them.

\paragraph{Mixtures of LoRA and rank-one experts.}
MoE-LoRA methods route tokens to low-rank experts~\citep{zadouri2024pushing,dou2024loramoe,li2024mixlora,wu2024mole,tian2024hydralora,zhuang2025ldmole,fan2025goat,mu2026talklora}. Finer experts help: T-REX and MoDE mix rank-one experts~\citep{zhang2024trex,ning2024mode}, SMoRA treats every rank as an expert~\citep{zhao2025smora}, FlyLoRA routes by the top-$k$ responses of a frozen random down-projection~\citep{zou2026flylora}, and MoRAM lets rank-one memories self-activate through their keys~\citep{lu2026moram}. Dynamic MoE routing varies the number of experts through top-$p$ rules, null experts, ReLU gates, or confidence~\citep{huang2024harder,zeng2024adamoe,wang2025remoe,jin2025dtopp,saliencro2026care}. \method{} shares self-routing with FlyLoRA and MoRAM but derives the number of active components from a noise model, the threshold from a risk estimate, and the combination rule from the joint sparse code.

\paragraph{Shrinkage and sparse codes in neural networks.}
Soft thresholding and iterative shrinkage are classical tools for sparse estimation~\citep{donoho1995denoising,daubechies2004iterative,beck2009fista}, and unrolled or competitive variants learn sparse codes quickly~\citep{rozell2008sparse,chen2018lista,monga2021unrolling}. In language models, thresholded activations yield activation sparsity~\citep{liu2025teal,lee2024cats,zhang2025rsparse} and sparse autoencoders~\citep{cunningham2024sparse,gao2025scaling,rajamanoharan2024jumprelu}. These methods fix a sparsity level or learn a threshold; TEAL, for instance, sets one magnitude cutoff per tensor to meet a target sparsity~\citep{liu2025teal}. \method{} instead derives the threshold from a noise model and applies it to the components of an adapter, so the cutoff follows each token's own noise floor. Extended related work is in Appendix~\ref{app:related}.

\section{\method{}}
\label{sec:method}

\subsection{Rank-one components as a code}
\label{sec:method-code}

A LoRA layer computes $y=W_0x+sBAx$ with $A\in\R^{r\times d}$, $B\in\R^{d_{\mathrm{out}}\times r}$ and scale $s$~\citep{hu2022lora}. Writing $a_i^{\!\top}$ for the rows of $A$ and $b_i$ for the columns of $B$,
\begin{equation}
sBAx=s\textstyle\sum_{i=1}^{r}b_i\,(a_i^{\!\top}x)=sBu,
\label{eq:code}
\end{equation}
so LoRA forms a \emph{code} $u=Ax$ and synthesizes the update from the value dictionary $B$. The code is dense: every component contributes to every token. We call a component \emph{active} for $x$ if its coefficient is nonzero, and the number of active components is the token's \emph{rank}, since the update restricted to a neighbourhood of $x$ has rank at most that number (Proposition~\ref{prop:lora}). \method{} replaces the dense code $u$ by a sparse code $z(x)$ in three steps (Figure~\ref{fig:overview}, Algorithm~\ref{alg:squelch}). We normalize the keys~\citep{buyukakyuz2024olora}, $\hat a_i=a_i/\|a_i\|$, so that responses $u=\hA x$ are comparable across components; this does not restrict the hypothesis class, since the scale of $a_i$ can move into $b_i$.

\subsection{Detect: the noise floor}
\label{sec:method-floor}

For a component that is irrelevant to $x$, the response $\hat a_i^{\!\top}x$ behaves like a random projection of $x$: approximately Gaussian with a scale common to all such components. The few relevant components stand out above it (Figure~\ref{fig:motivation}a). We estimate the scale of this floor per token with the median absolute response~\citep{hampel1974influence,donoho1995denoising},
\begin{equation}
\hs(x)=\operatorname{med}_i|u_i(x)|\,/\,\Phi^{-1}(\tfrac34),
\label{eq:mad}
\end{equation}
which ignores up to half of the components carrying signal, and soft-threshold the responses:
\begin{equation}
z^0=\soft_{\tau(x)}(u),\qquad \tau(x)=c_\ell\,\hs(x),
\label{eq:squelch}
\end{equation}
where $\soft_\tau(v)=\mathrm{sign}(v)\max(|v|-\tau,0)$ acts elementwise and $c_\ell$ is a per-module multiple. A component survives only if it clears the token's own floor. Tokens with strong, structured responses keep many components, weak tokens keep few, and a token whose responses are all floor keeps none and passes through the frozen layer unchanged. Soft thresholding is continuous in $u$, so gradients reach every surviving component. Because every response belongs to the floor at initialization, the first $5\%$ of steps ramp $c_\ell$ linearly from zero, i.e.\ train as a dense LoRA.

\subsection{Decide: a budget from an unbiased risk}
\label{sec:method-sure}

The multiple $c_\ell$ sets the expected rank of module $\ell$, and fitting it to the task loss drives it toward zero, since denser codes fit the training loss more easily (Table~\ref{tab:ablation}). We instead treat $c_\ell$ as a denoising parameter. For standardized responses $v=u/\hs$, Stein's lemma gives an unbiased estimate of the squared error of $\soft_c(v)$ with respect to the unknown signal,
\begin{equation}
\SURE(c;v)=r-2N_c(v)+\textstyle\sum_i\min(v_i^2,c^2),
\label{eq:sure}
\end{equation}
where $N_c(v)=\#\{i:|v_i|\le c\}$, and it requires no labels~\citep{stein1981estimation,donoho1995adapting}. Every $20$ steps, each module sets $\hat c_\ell=\arg\min_{c\in[0,\cu]}\sum_{x}\SURE(c;v(x))$ over the tokens of the batch and updates $c_\ell$ by an exponential moving average. The cap $\cu=\sqrt{2\log r}$ is the universal threshold, and, as in SureShrink, $\cu$ is used directly when the responses show no signal above the floor~\citep{donoho1995adapting}. Each module's rank budget therefore emerges from its own responses, and no global budget or per-module rank is specified.

\subsection{Explain away: from marginal to joint decisions}
\label{sec:method-ea}

Equation~\eqref{eq:squelch} decides for each component in isolation. When the values of co-active components overlap, these marginal decisions double count shared directions and miss components that are needed only in combination (Figure~\ref{fig:motivation}d). The code that accounts for the overlap solves a weighted LASSO over the value dictionary,
\begin{equation}
\min_z\ \tfrac12\|Bz\|^2-(Du)^{\!\top}z+\tau\textstyle\sum_iD_{ii}|z_i|,
\label{eq:lasso}
\end{equation}
where $G=B^{\!\top}B$ and $D=\mathrm{diag}(G)$. This is the LASSO for a target update whose per-component regression coefficients are $u$. If the values are orthogonal, \eqref{eq:lasso} decouples and its solution is exactly $z^0$. \method{} takes one proximal-gradient step from $z^0$,
\begin{equation}
z=\soft_{\eta\tau}\big(z^0+\eta\,(u-(I+M)z^0)\big),
\label{eq:ea}
\end{equation}
with $M=D^{-1}(G-D)$ and $\eta=1/\lambda_{\max}(D^{-1/2}GD^{-1/2})$. The term $Mz^0$ \emph{explains away} what co-active components already cover: a component whose value is aligned with a stronger co-active one is suppressed, and one whose contribution complements them can be recruited. The step uses only the value Gram, costs $O(r^2)$ per token, and adds no parameters. The output of the layer is $y=W_0x+sBz$.

\begin{algorithm}[t]
\footnotesize
\DontPrintSemicolon
\SetAlgoLined
\caption{\method{} layer (training step)}
\label{alg:squelch}
\KwInput{tokens $X$; $W_0$; $A,B$; multiple $c_\ell$; step $t$}
$u\leftarrow \hA x$ for each $x\in X$ \tcp*{unit-norm keys}
$\hs(x)\leftarrow\operatorname{med}_i|u_i|/0.6745$\;
$\tau(x)\leftarrow\min(1,t/T_w)\,c_\ell\,\hs(x)$ \tcp*{warm-up ramp}
$z^0\leftarrow\soft_{\tau}(u)$ \tcp*{squelch, Eq.~\eqref{eq:squelch}}
$G\leftarrow B^{\!\top}B$;\quad $M\leftarrow D^{-1}(G-D)$\;
$\eta\leftarrow1/\lambda_{\max}(D^{-1/2}GD^{-1/2})$\;
$z\leftarrow\soft_{\eta\tau}\big(z^0+\eta(u-(I{+}M)z^0)\big)$ \tcp*{Eq.~\eqref{eq:ea}}
$y\leftarrow W_0x+sBz$\;
\If{$t\bmod 20=0$}{$\hat c\leftarrow\arg\min_{c\le\cu}\sum_x\SURE(c;u/\hs)$\;
$c_\ell\leftarrow0.9\,c_\ell+0.1\,\hat c$\;}
\end{algorithm}

\paragraph{Cost.}
\method{} has exactly LoRA's parameters. Per token it adds a median over $r$ responses and an $r\times r$ product; per module and step it adds $B^{\!\top}B$ and a power iteration on an $r\times r$ matrix. Inactive components skip the up-projection. Like every mixture of LoRA experts, the update depends on the input and cannot be merged into $W_0$.

\section{Theory}
\label{sec:theory}

The statements below idealize the responses as a sparse signal plus Gaussian noise; proofs are in Appendix~\ref{app:proofs}, and Table~\ref{tab:numerics} checks every statement numerically.

\begin{proposition}[LoRA as a special case; conic structure]
\label{prop:lora}
(i) With $c_\ell=0$ and no explaining away, \method{} equals LoRA. (ii) For $\alpha>0$, $z(\alpha x)=\alpha z(x)$: active sets depend only on $x/\|x\|$, and the update is a continuous, positively homogeneous, piecewise-linear function of $x$. (iii) On each piece, the update has rank at most $|\supp z(x)|$.
\end{proposition}

\begin{proposition}[Near-oracle detection]
\label{prop:oracle}
Let $u=\theta+\sigma\varepsilon$ with $\varepsilon\sim\mathcal N(0,I_r)$, and let $\tau_U=\sigma\sqrt{2\log r}$. (a) $\E\|\soft_{\tau_U}(u)-\theta\|^2\le(2\log r+1)\big(\sigma^2+\sum_i\min(\theta_i^2,\sigma^2)\big)$, where the sum is the risk of the oracle that keeps component $i$ iff $|\theta_i|>\sigma$. (b) No component with $\theta_i=0$ is active with probability at least $1-1/\sqrt{\pi\log r}$, and each component with $|\theta_i|\ge2\tau_U$ is active with probability at least $1-1/(r\sqrt{\pi\log r})$.
\end{proposition}

\begin{proposition}[Budget-free threshold]
\label{prop:sure}
Let $v_n=\theta_n+\varepsilon_n$, $n\le N$, with independent $\varepsilon_n\sim\mathcal N(0,I_r)$. (a) For every $c\ge0$, $\E\sum_n\SURE(c;v_n)=\E\sum_n\|\soft_c(v_n)-\theta_n\|^2$. (b) $\sup_{c\in[0,\cu]}|\sum_n(\SURE(c;v_n)-\|\soft_c(v_n)-\theta_n\|^2)|/(Nr)\to0$ in probability, so the risk of $\hat c=\arg\min\sum_n\SURE(c;v_n)$ exceeds the best risk on $[0,\cu]$ by $o_p(Nr)$.
\end{proposition}

\begin{proposition}[Explaining away]
\label{prop:ea}
Let $F$ be the objective in \eqref{eq:lasso}. (a) Equation~\eqref{eq:ea} is a proximal-gradient step on $F$ in the metric $D$, so $F(z)\le F(z^0)$, with equality iff $z^0$ minimizes $F$. (b) If two co-active components have identical values and responses $u>\tau$, the first pass outputs $2(u-\tau)b$ while \eqref{eq:ea} outputs $(u-\tau)b$, the LASSO solution. (c) For $i\ne j$, $\partial\mathcal L/\partial G_{ij}=0$ unless one of the two components is active in $z^0$ and the other is active in $z$. (d) With unit-norm values, and when $z^0$, $z$, and the minimizer $z^\star$ share their support $S$ and signs, $\|z^0-z^\star\|\le(|S|-1)\mu_S\|z^\star\|$ and $\|z-z^\star\|\le\|I-\eta G_{SS}\|_2\,\|z^0-z^\star\|$, where $\mu_S$ is the largest coherence $|b_i^{\!\top}b_j|$ within $S$.
\end{proposition}

Proposition~\ref{prop:oracle} says that the squelch does almost as well as an oracle that knows which components a token needs, paying only the $\log r$ factor that no estimator can avoid uniformly~\citep{donoho1994ideal}. Proposition~\ref{prop:sure} replaces the rank budget with a quantity that can be estimated from the responses alone. Proposition~\ref{prop:ea} explains the third step: marginal decisions are exact only when co-active values are incoherent; otherwise their error grows with $(|S|-1)\mu_S$, and one explaining-away step contracts it. Part~(c) makes the step \emph{co-activation-aware}: the task gradient reaches the overlap of two values only through tokens on which both are active.

\begin{table*}[t]
\centering
\footnotesize
\setlength{\tabcolsep}{3.9pt}
\renewcommand{\arraystretch}{0.97}
\begin{tabular}{l c c cccccccc c}
\rowcolor{tabhead}
\Th{Method} & \Th{Params\,(\%)} & \Th{Rank} & \Th{BoolQ} & \Th{PIQA} & \Th{SIQA} & \Th{HellaS.} & \Th{WinoG.} & \Th{ARC-e} & \Th{ARC-c} & \Th{OBQA} & \Th{Avg.} \\
\multicolumn{12}{@{}l}{\textit{LLaMA-3.1-8B}}\\
LoRA & 1.04 & 32 & 72.9 & 88.0 & 80.1 & 94.1 & 85.0 & 90.3 & 79.6 & 85.2 & 84.40 \\
DoRA & 1.05 & 32 & 73.5 & 88.4 & 80.6 & 94.6 & 85.6 & 90.8 & 80.6 & 85.9 & 85.00 \\
AdaLoRA & 1.04 & 32 & 73.2 & 88.3 & 80.5 & 94.3 & 85.2 & 90.6 & 79.8 & 85.7 & 84.70 \\
GoRA & 1.04 & 32 & 73.9 & 88.5 & 80.7 & 94.7 & 86.1 & 90.9 & 80.6 & 86.4 & 85.23 \\
\addlinespace[2pt]
HydraLoRA & 1.05 & 16 & 74.0 & 88.8 & 80.8 & 94.8 & 86.1 & 91.1 & 81.0 & 86.4 & 85.38 \\
MixLoRA & 1.07 & 8 & 74.0 & 89.0 & 81.2 & 94.7 & 86.5 & 91.4 & 81.5 & 87.0 & 85.66 \\
AdaMoLE & 1.07 & 11.6 & 73.8 & 88.9 & 81.0 & 94.6 & 86.5 & 91.4 & 81.2 & 86.7 & 85.51 \\
LD-MoLE & 1.07 & 10.8 & 74.6 & 89.2 & 81.6 & 95.1 & 86.9 & 91.8 & 81.9 & 87.3 & 86.05 \\
SMoRA & 1.53 & 8 & 74.5 & 89.0 & 81.6 & 94.7 & 86.6 & 91.6 & 81.6 & 87.0 & 85.83 \\
FlyLoRA & 0.55 & 8 & 74.6 & 89.4 & 81.8 & 95.1 & 86.8 & 91.7 & 82.2 & 87.3 & 86.11 \\
\addlinespace[2pt]
\ourname{\method{} (ours)} & 1.04 & 8.6 & \textbf{75.3} & \textbf{90.3} & \textbf{82.5} & \textbf{95.2} & \textbf{88.1} & \textbf{92.6} & \textbf{83.4} & \textbf{88.9} & \textbf{87.04} \\
\addlinespace[3pt]
\multicolumn{12}{@{}l}{\textit{Qwen3-8B}}\\
LoRA & 1.07 & 32 & 75.6 & 89.9 & 81.3 & 95.2 & 86.1 & 94.3 & 86.0 & 88.4 & 87.10 \\
DoRA & 1.08 & 32 & 75.9 & 90.1 & 81.7 & 95.3 & 86.6 & 94.6 & 86.6 & 89.0 & 87.48 \\
AdaLoRA & 1.07 & 32 & 75.7 & 90.1 & 81.5 & 95.2 & 86.3 & 94.5 & 86.5 & 88.9 & 87.34 \\
GoRA & 1.07 & 32 & 76.3 & 90.4 & 82.0 & 95.7 & 86.9 & 94.9 & 86.8 & 89.1 & 87.76 \\
\addlinespace[2pt]
HydraLoRA & 1.08 & 16 & 76.2 & 90.5 & 82.2 & 95.8 & 87.1 & 94.9 & 87.0 & 89.3 & 87.88 \\
MixLoRA & 1.10 & 8 & 76.4 & 90.9 & 82.4 & 95.7 & 87.2 & 95.1 & 87.5 & 89.8 & 88.13 \\
AdaMoLE & 1.10 & 11.9 & 76.6 & 90.7 & 82.1 & 95.7 & 87.2 & 95.3 & 87.5 & 89.7 & 88.10 \\
LD-MoLE & 1.10 & 11.2 & 77.1 & 91.1 & 82.7 & 95.8 & 87.4 & 95.4 & 87.7 & 90.0 & 88.40 \\
SMoRA & 1.57 & 8 & 76.9 & 91.0 & 82.6 & \textbf{96.0} & 87.5 & 95.5 & 87.8 & 89.9 & 88.40 \\
FlyLoRA & 0.57 & 8 & 76.9 & 91.1 & 82.4 & 95.8 & 87.8 & 95.6 & 88.2 & 90.3 & 88.51 \\
\addlinespace[2pt]
\ourname{\method{} (ours)} & 1.07 & 8.8 & \textbf{77.8} & \textbf{91.8} & \textbf{83.5} & 95.9 & \textbf{88.5} & \textbf{96.0} & \textbf{89.4} & \textbf{91.3} & \textbf{89.28} \\
\tabend
\end{tabular}
\caption{\textbf{Commonsense reasoning} (accuracy, \%, mean of three seeds). \textbf{Params}: trainable parameters relative to the backbone. \textbf{Rank}: rank-one components that contribute to a token's update, averaged over tokens and modules. Single adapters, then mixtures. \textbf{Bold}: best score. \method{} is in the header color.}
\label{tab:main-cs}
\end{table*}

\begin{table*}[t]
\centering
\footnotesize
\setlength{\tabcolsep}{3.4pt}
\renewcommand{\arraystretch}{0.97}
\resizebox{\textwidth}{!}{%
\begin{tabular}{l ccccccc ccccccc}
\rowcolor{tabhead}
& \multicolumn{7}{c}{\Th{LLaMA-3.1-8B}} & \multicolumn{7}{c}{\Th{Qwen3-8B}} \\
\rowcolor{tabhead}
\Th{Method} & \Th{MMLU} & \Th{SciQA} & \Th{GSM8K} & \Th{MATH} & \Th{HumanE.} & \Th{MBPP} & \Th{Avg.} & \Th{MMLU} & \Th{SciQA} & \Th{GSM8K} & \Th{MATH} & \Th{HumanE.} & \Th{MBPP} & \Th{Avg.} \\
LoRA & 63.8 & 93.6 & 71.2 & 24.6 & 42.1 & 51.3 & 57.77 & 76.1 & 95.2 & 89.3 & 58.4 & 72.6 & 66.9 & 76.42 \\
DoRA & 64.4 & 93.7 & 71.7 & 24.8 & 43.0 & 52.0 & 58.27 & 76.5 & 95.3 & 89.4 & 59.2 & 73.3 & 67.4 & 76.85 \\
GoRA & 64.5 & 93.8 & 72.4 & 25.5 & 43.6 & 52.5 & 58.72 & 76.3 & 95.7 & 90.1 & 59.8 & 74.0 & 68.0 & 77.32 \\
\addlinespace[2pt]
HydraLoRA & 64.5 & 93.9 & 72.3 & 25.7 & 43.6 & 52.1 & 58.68 & 76.7 & 95.7 & 89.9 & 59.4 & 73.7 & 67.9 & 77.22 \\
MixLoRA & 64.3 & 94.3 & 73.2 & 26.0 & 44.2 & 53.0 & 59.17 & 76.6 & 95.7 & 90.1 & 59.7 & 74.3 & 67.8 & 77.37 \\
LD-MoLE & 64.8 & 94.5 & 73.5 & 26.0 & 44.5 & 53.1 & 59.40 & 76.7 & 95.9 & 90.8 & 60.4 & 74.5 & 68.5 & 77.80 \\
SMoRA & 64.4 & 94.6 & 73.0 & 25.9 & 44.6 & 53.0 & 59.25 & 76.7 & 95.9 & 90.2 & 60.4 & 74.4 & 68.1 & 77.62 \\
FlyLoRA & \textbf{65.1} & 94.6 & 73.7 & 26.4 & 45.2 & 53.4 & 59.73 & 76.8 & 95.8 & 90.5 & 61.0 & 75.2 & 68.6 & 77.98 \\
\addlinespace[2pt]
\ourname{\method{} (ours)} & 65.0 & \textbf{95.3} & \textbf{75.7} & \textbf{27.6} & \textbf{47.3} & \textbf{54.6} & \textbf{60.92} & \textbf{77.0} & \textbf{96.0} & \textbf{91.2} & \textbf{62.0} & \textbf{76.6} & \textbf{69.6} & \textbf{78.73} \\
\tabend
\end{tabular}}
\caption{\textbf{Knowledge, science, math, and code} (accuracy or pass@1, \%, mean of three seeds; one adapter per task). Parameter budgets and active ranks as in Table~\ref{tab:main-cs}.}
\label{tab:main-dom}
\end{table*}

\section{Experiments}
\label{sec:experiments}

\subsection{Setup}
\label{sec:setup}

\paragraph{Models and tasks.}
We adapt LLaMA-3.1-8B~\citep{grattafiori2024llama3} and Qwen3-8B-Base~\citep{yang2025qwen3}, and use Qwen3-0.6B/1.7B/4B/14B-Base for scaling. \emph{Commonsense reasoning} trains on Commonsense170K and evaluates on BoolQ, PIQA, SIQA, HellaSwag, WinoGrande, ARC-e/c, and OBQA~\citep{hu2023llmadapters,clark2019boolq,bisk2020piqa,sap2019siqa,zellers2019hellaswag,sakaguchi2020winogrande,clark2018arc,mihaylov2018obqa}. The \emph{domain} suite trains one adapter per task: MMLU with its auxiliary training set~\citep{hendrycks2021mmlu}, ScienceQA~\citep{lu2022scienceqa}, math on a 100K subset of MetaMathQA evaluated on GSM8K and MATH~\citep{yu2024metamath,cobbe2021gsm8k,hendrycks2021math}, and code on a 100K subset of CodeFeedback evaluated on HumanEval and MBPP~\citep{zheng2024opencodeinterpreter,chen2021humaneval,austin2021mbpp}. GLUE with DeBERTaV3-base~\citep{wang2019glue,he2023debertav3}, joint multi-task training, and forgetting are reported in Appendix~\ref{app:extra}.

\paragraph{Baselines.}
Single adapters: LoRA, DoRA~\citep{liu2024dora}, AdaLoRA, and GoRA. Mixtures: HydraLoRA, MixLoRA, AdaMoLE, LD-MoLE, SMoRA, and FlyLoRA. All adapters target the seven projection matrices and are matched to LoRA with $r{=}32$ in trainable parameters, except FlyLoRA, whose frozen down-projection halves them, and SMoRA, whose router adds a third. Rank-wise baselines use $32$ components with $k{=}8$ active, the configuration of FlyLoRA~\citep{zou2026flylora}.

\paragraph{Implementation.}
\method{} uses $r{=}32$, $\alpha{=}64$, a $5\%$ dense warm-up, SURE every $20$ steps with EMA $0.9$, and cap $\cu{=}\sqrt{2\log 32}$, with identical settings for every task and backbone. All methods share the optimizer, schedule, and data; we report the mean of three seeds (standard deviations and full settings in Appendices~\ref{app:setup} and~\ref{app:extra}).

\subsection{Main Results}
\label{sec:main}

\begin{figure*}[t]
\centering
\includegraphics[width=\textwidth]{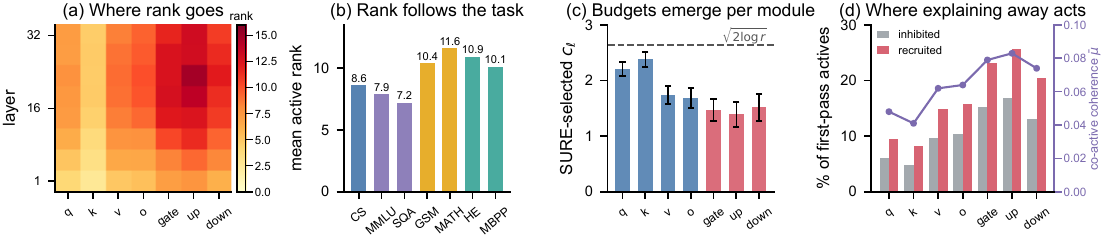}
\caption{\textbf{Where \method{} spends rank} (LLaMA-3.1-8B). (a) Mean active rank by layer and module: FFN projections in the middle and upper layers use the most components. (b) Mean active rank by task: math and code use more components than knowledge and science. (c) SURE-selected multiples $c_\ell$ ($\pm$ s.d.\ across layers) stay below the universal cap and are lower for FFN modules. (d) Explaining away inhibits and recruits the most components in FFN modules, where co-active values are most coherent (right axis).}
\label{fig:mechanism}
\end{figure*}

\begin{figure*}[t]
\centering
\includegraphics[width=\textwidth]{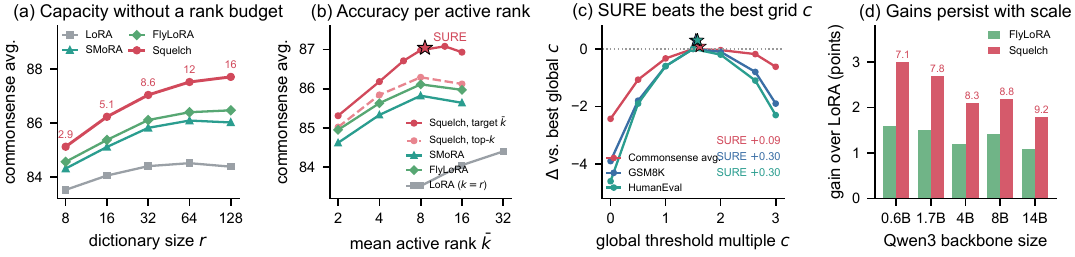}
\caption{\textbf{Capacity, budget, and threshold selection} (commonsense unless stated). (a) Accuracy vs.\ dictionary size $r$; numbers give \method{}'s mean active rank. (b) Accuracy vs.\ mean active rank: \method{} with a target rank (budgeted variant), top-$k$ variants, and LoRA; the star is SURE's choice. (c) Accuracy of global fixed multiples $c$ relative to the best one; stars mark per-module SURE. (d) Gain over LoRA across Qwen3 sizes; numbers give \method{}'s mean active rank.}
\label{fig:results}
\end{figure*}

\method{} attains the best average in all four settings (Tables~\ref{tab:main-cs} and~\ref{tab:main-dom}). On commonsense reasoning it reaches \oursCsL{} on LLaMA-3.1-8B and \oursCsQ{} on Qwen3-8B, $+\gainLoraCsL$ and $+\gainLoraCsQ$ over LoRA and $+\gainBestCsL$ and $+\gainBestCsQ$ over the strongest baseline, \bestNameCsL{}. On the domain suite, the margins over \bestNameDomL{} are $+\gainBestDomL$ and $+\gainBestDomQ$. Three patterns stand out. First, input-adaptive methods beat input-shared ones at comparable parameter budgets. LD-MoLE, SMoRA, and FlyLoRA all exceed AdaLoRA and GoRA, so where rank is spent matters more than how it is divided among modules. Second, \method{} is better than the rank-wise mixtures at a similar active rank ($\rankOurs$ vs.\ $8$) and without a router: the gain comes from choosing \emph{which and how many} components fire, not from activating more. Third, gains concentrate on tasks that require composition, namely ARC-c, WinoGrande, OBQA, GSM8K, and code. On MMLU, whose adaptation is closer to a uniform shift, FlyLoRA and \method{} are within $0.1$ on LLaMA, and on HellaSwag SMoRA is best on Qwen3; both are near saturation. The improvements exceed three seed standard deviations in every setting (Table~\ref{tab:seeds}).

\subsection{Ablations}
\label{sec:ablation}

Table~\ref{tab:ablation} removes one ingredient at a time. \emph{The floor must be per token}: a per-layer noise scale loses $0.46$, and a non-robust RMS estimate, inflated by the signal itself, over-squelches ($\bar k{=}5.7$). \emph{The threshold must be decided, not learned}: learning $c_\ell$ by the task loss drives it to near zero ($\bar k{=}27.9$) and falls back close to LoRA ($85.08$); the universal threshold without SURE is too conservative ($\bar k{=}5.1$, $-0.27$); a fixed top-$k$ at the same average rank loses $0.73$. \emph{Explaining away matters most among the three steps}: without it, the code is $\gapNoEA$ from the joint code instead of $\gapOurs$, the squelch keeps fewer components ($\bar k{=}\rankNoEA$), and accuracy drops by $\dropNoEA$. Blocking the gradient through $B^{\!\top}\!B$ costs $0.18$, the co-activation-aware credit of Proposition~\ref{prop:ea}(c). A penalty on co-active coherence partially substitutes for explaining away but adds nothing on top of it. The dense warm-up is necessary: without it, a third of the components never clear the floor.

\begin{table}[t]
\centering
\footnotesize
\setlength{\tabcolsep}{2.2pt}
\renewcommand{\arraystretch}{0.97}
\resizebox{\columnwidth}{!}{%
\begin{tabular}{l c c c c}
\rowcolor{tabhead}
\Th{Configuration} & \Th{CS} & \Th{Dom.} & \Th{$\bar k$} & \Th{Gap} \\
\ourname{\method{} (full)} & \textbf{87.04} & \textbf{60.92} & 8.6 & 0.09 \\
\addlinespace[2pt]
Per-layer EMA noise scale (not per token) & 86.58\,{\scriptsize$(-0.46)$} & 60.21 & 8.9 & 0.10 \\
RMS instead of MAD noise scale & 86.64\,{\scriptsize$(-0.40)$} & 60.30 & 5.7 & 0.08 \\
\addlinespace[2pt]
Universal $c{=}\sqrt{2\log r}$ (no SURE) & 86.77\,{\scriptsize$(-0.27)$} & 60.46 & 5.1 & 0.07 \\
$c$ learned by the task loss & 85.08\,{\scriptsize$(-1.96)$} & 58.62 & 27.9 & 0.31 \\
Top-$k$ at matched $k{=}9$ (no floor) & 86.31\,{\scriptsize$(-0.73)$} & 59.94 & 9.0 & --- \\
Hard instead of soft threshold & 86.85\,{\scriptsize$(-0.19)$} & 60.52 & 8.3 & 0.11 \\
\addlinespace[2pt]
No explaining away (marginal code) & 86.43\,{\scriptsize$(-0.61)$} & 60.05 & 6.9 & 0.27 \\
EA, stop-gradient through $B^{\!\top}\!B$ & 86.86\,{\scriptsize$(-0.18)$} & 60.63 & 8.4 & 0.09 \\
Coherence penalty instead of EA & 86.55\,{\scriptsize$(-0.49)$} & 60.19 & 6.8 & 0.19 \\
EA $+$ coherence penalty & 87.01\,{\scriptsize$(-0.03)$} & 60.89 & 8.5 & 0.07 \\
\addlinespace[2pt]
No dense warm-up & 85.69\,{\scriptsize$(-1.35)$} & 58.95 & 6.2 & 0.12 \\
Unnormalized keys & 86.49\,{\scriptsize$(-0.55)$} & 60.11 & 8.1 & 0.13 \\
\tabend
\end{tabular}}
\caption{\textbf{Ablations} on LLaMA-3.1-8B. CS/Dom.: commonsense/domain averages; $\bar k$: mean active rank; Gap: $\|z-z^\star\|/\|z^\star\|$ to the joint code at the same threshold.}
\label{tab:ablation}
\end{table}

\subsection{Mechanism Analysis}
\label{sec:mechanism}

\paragraph{Rank follows the input.}
Per-token ranks range from $0$ to over $20$ (Figure~\ref{fig:tokens}). Numerals and operators use $11.8$ components on average, content words $9.4$, function words $5.3$, and punctuation $3.9$. Across tasks, math and code use the most (Figure~\ref{fig:mechanism}b). Out of domain, the squelch closes more often: on MMLU prompts, a commonsense adapter leaves \rankZeroOOD\% of token--module pairs without active components, against \rankZeroIn\% in domain. This selectivity explains the smaller forgetting of \method{} ($-\forgetOurs$ vs.\ $-\forgetLora$ points for LoRA on held-out benchmarks; Table~\ref{tab:forget}).

\paragraph{Budgets emerge per module.}
SURE assigns lower multiples, and thus more components, to FFN projections than to attention queries and keys (Figure~\ref{fig:mechanism}c), and the resulting ranks peak in the upper-middle layers (Figure~\ref{fig:mechanism}a). The selected multiples stabilize within the first tenth of training (Figure~\ref{fig:dynamics}) and stay below the universal cap everywhere, so the cap acts only as a safeguard.

\paragraph{Explaining away acts where values overlap.}
Explaining away removes $5$--$17\%$ of first-pass components and recruits $8$--$26\%$ more, most in FFN modules, where the values of co-active components are most coherent (Figure~\ref{fig:mechanism}d). Restricting it to attention modules recovers $0.21$ of its $0.61$-point gain, and restricting it to FFN modules recovers $0.49$, in line with Proposition~\ref{prop:ea}(d). A controlled synthetic study with a known per-input rank (Appendix~\ref{app:synthetic}) shows the same division of labour: the squelch alone is conservative, and explaining away makes the active count track the true rank.

\subsection{Capacity, Budget, and Scaling}
\label{sec:scaling}

\paragraph{Larger dictionaries keep helping.}
Because each token uses only what it needs, \method{} converts a larger dictionary into accuracy without a larger per-token rank budget. From $r{=}8$ to $128$ it gains $2.6$ points while its mean active rank grows only from $2.9$ to $15.8$; LoRA saturates beyond $r{=}32$ and the top-$k$ mixtures around $r{=}64$ (Figure~\ref{fig:results}a).

\paragraph{The SURE choice sits at the knee.}
With a target average rank enforced by a dual variable, \method{} dominates top-$k$ rules at every budget, including its own top-$k$ variant (Figure~\ref{fig:results}b). The budget-free SURE point ($\bar k{=}8.6$) lies at the knee of this frontier, and per-module SURE beats the best \emph{global} multiple found by grid search on three tasks (Figure~\ref{fig:results}c). The budget is thus found, not tuned.

\paragraph{Scale and cost.}
The gain over LoRA persists from $0.6$B to $14$B parameters, and the mean active rank grows slowly with model size (Figure~\ref{fig:results}d). \method{} trains $\trainOverhead\%$ slower than unmerged LoRA and adds $\inferOverhead\%$ to inference latency, less than mixtures with learned routers (Table~\ref{tab:overhead}).

\paragraph{Beyond single tasks.}
The same selectivity helps when tasks share one adapter. Trained jointly on the six domain tasks, \method{} loses only \mtDropOurs{} points relative to separate adapters, against \mtDropLora{} for LoRA and \mtDropFly{} for FlyLoRA (Table~\ref{tab:multitask}), because tokens of different tasks clear the floor on different components. On GLUE with DeBERTaV3-base, it raises the average from \glueLora{} (LoRA) to \glueOurs{} with $2.9$ of $8$ components active per token (Table~\ref{tab:glue}).

\section{Conclusion}
\label{sec:conclusion}

How many rank-one components a token needs is a detection problem: a token's responses to a LoRA dictionary are a sparse signal on a noise floor. \method{} answers it with three classical steps, a robust per-token floor, a threshold chosen by unbiased risk, and one explaining-away step toward the joint sparse code. It needs neither a router nor a rank budget, adds no parameters to LoRA, and improves over strong single-adapter and mixture baselines across tasks and scales. The same view applies wherever a model combines many small components per input, including mixtures of full experts and sparse autoencoders.

\section*{Limitations}

The noise model is an idealization: responses of irrelevant components are only approximately Gaussian, and the median estimate is biased upward when many components are active (Appendix~\ref{app:more-theory}), which SURE compensates only in part. Explaining away is a single proximal step; it closes part of the gap to the joint code, and our synthetic study shows that it resolves redundancy in the code rather than disentangling the dictionary itself. Like every input-dependent adapter, \method{} cannot be merged into the backbone, and it adds a small inference overhead. Our experiments cover decoder models up to 14B parameters and one encoder, supervised fine-tuning, and English benchmarks; reinforcement-learning fine-tuning, multilingual adaptation, and larger models remain to be tested.

\section*{Ethics Statement}

\method{} is a general fine-tuning method and inherits the risks of the models and data to which it is applied, including biases in instruction data and the lower cost of specializing models for harmful uses. All datasets are public and used under their licenses; no human subjects or personal data are involved. Because \method{} leaves out-of-domain tokens closer to the frozen model, it reduces but does not remove the risk that fine-tuning erodes safety behavior; such behavior should be evaluated after adaptation. AI assistants were used for language editing and for checking LaTeX; every reference was checked against its primary record.

\bibliography{custom}

\appendix
\newcommand{\synDupCosNoEA}{0.17}
\newcommand{\synDupCosEA}{0.39}
\newcommand{\synDupPrNoEA}{0.87}
\newcommand{\synDupPrEA}{0.33}

\section{Extended Related Work}
\label{app:related}

\paragraph{LoRA and its reparameterizations.}
LoRA~\citep{hu2022lora} and its quantized form~\citep{dettmers2023qlora} made low-rank updates the default for adapting large models. Later work changes the parameterization or the optimization of the update. DoRA separates magnitude from direction~\citep{liu2024dora}; PiSSA initializes from principal singular directions~\citep{meng2024pissa}; LoRA-GA and LoRA-Pro align the low-rank gradient with the full one~\citep{wang2024loraga,wang2025lorapro}; rsLoRA and LoRA+ correct the scale and learning rates of the two factors~\citep{kalajdzievski2023rslora,hayou2024loraplus}; VeRA shares frozen random factors~\citep{kopiczko2024vera}; MELoRA and MoRA raise the rank of the update at a fixed parameter count~\citep{ren2024melora,jiang2024mora}. Two findings are relevant to our noise-floor view. LoRA's down-projection changes much less than its up-projection and can even be frozen at random initialization~\citep{zhang2023lorafa,zhu2024asymmetry}; the responses $Ax$ of most components therefore stay close to random projections of the input, which is the floor that \method{} estimates. LoRA also forgets less than full fine-tuning~\citep{biderman2024lora}, but its updates introduce new high-variance directions that account for much of the forgetting~\citep{shuttleworth2024illusion,xie2026intruder}. \method{} leaves the tokens on which no component clears the floor on the frozen path.

\paragraph{Rank allocation.}
Static allocation gives each module its own rank. AdaLoRA prunes singular values by sensitivity~\citep{zhang2023adalora}, SoRA gates ranks with a proximal $\ell_1$ step, i.e.\ soft thresholding of a gate vector that is shared by all inputs~\citep{ding2023sora}, and DyLoRA and IncreLoRA train nested or growing ranks~\citep{valipour2023dylora,zhang2023increlora}. ALoRA and AutoLoRA learn allocations by ablation scores or meta-learning~\citep{liu2024alora,zhang2024autolora}. Recent methods use gradient information~\citep{he2025gora}, beam search over ranks~\citep{gu2025beamlora}, elastic pruning and growth~\citep{chang2025elalora}, differentiable module selection~\citep{jiang2025diffora}, hypernetworks~\citep{zhang2025hyperadalora}, entropy~\citep{liu2026flexlora}, per-expert ranks for MoE backbones~\citep{deng2026drlora}, information-bottleneck filtering~\citep{xiao2026structlora}, variational relevance~\citep{doyle2025lrvd}, integrated gradients with uncertainty~\citep{cui2026igulora}, calibration-time gradient variance~\citep{sathyavageeswaran2026fimlora}, or nested representations that support several ranks at test time~\citep{modoranu2026matryoshkalora}. Gradient-based importance can fail, for example under reinforcement learning~\citep{sawant2026gradient}. All of these methods fix a module's rank for every input. RaNA allocates ranks and neurons per input to speed up inference~\citep{garcia2025rana}, and Flexi-LoRA predicts a rank from the complexity of the input~\citep{li2026flexilora}; \method{} obtains the per-token rank from the responses themselves and requires neither a predictor nor a budget.

\paragraph{Mixtures of LoRA experts.}
LoRAMoE, MixLoRA, MoLE, and MoLA route tokens or layers to LoRA experts~\citep{dou2024loramoe,li2024mixlora,wu2024mole,gao2024mola}; HydraLoRA shares the down-projection across asymmetric heads~\citep{tian2024hydralora}; MoSLoRA mixes subspaces~\citep{wu2024moslora}; and extremely light variants mix vectors or rank-one factors~\citep{zadouri2024pushing,zhu2023sira}. Multi-task variants target task conflicts~\citep{liu2024moelora,zhang2025lori,shi2026samora}. Routing has been refined with hybrid, learnable, or dynamic rules~\citep{li2025dynmole,zhuang2025ldmole,liu2024adamole}, orthogonality or diversity among experts~\citep{feng2025omoe}, communication between experts~\citep{mu2026talklora}, rotational gates~\citep{sun2026rotmole}, singular-value alignment~\citep{fan2025goat}, and expert-aware adaptation of MoE backbones~\citep{yang2026moe2lora}. Other designs place experts in the attention projections~\citep{li2025loramixer}, partition adapters into blocks with adaptive routing~\citep{barazandeh2026localized}, compose frozen tensor-train experts~\citep{kunwar2025ttloramoe}, or remove the discrete router so that the mixture can be merged~\citep{han2025furina}. Routing alone does not remove interference inside a shared adapter subspace~\citep{chowdhury2026routing}.

\paragraph{Rank-one experts.}
Fine-grained experts are the closest line of work. T-REX combines rank-one experts with semantic priors for the router~\citep{zhang2024trex}; MoDE shares the down-projection and routes over rank-one ``dyadic'' experts~\citep{ning2024mode}; SMoRA shows that routing among several LoRAs equals rank partitioning within one LoRA and treats every rank as an expert with a learned router and top-$k$~\citep{zhao2025smora}; FlyLoRA freezes a sparse random down-projection whose top-$k$ responses select the active columns of the up-projection, an implicit router~\citep{zou2026flylora}; and MoRAM views rank-one adapters as key--value memories that self-activate by their keys~\citep{lu2026moram}. Rank-one experts also appear as orthogonal experts obtained from the singular vectors of pretrained weights~\citep{yuan2025moore}, as reusable rank-wise experts in federated fine-tuning~\citep{wu2025smartfed}, and as an expert pool inside a single LoRA for continual learning~\citep{fa2026rank1pool,zou2026fly}. Their view of self-routing through each component's own key goes back to reading feed-forward layers as key--value memories~\citep{geva2021transformer}. \method{} keeps self-routing and adds what these methods fix by hand: how many components fire (a noise-calibrated threshold instead of top-$k$), how the threshold is set (an unbiased risk estimate instead of a budget), and how co-active components combine (a joint code instead of a sum).

\paragraph{Dynamic routing in sparse MoE.}
Sparse MoE layers route each token to a fixed number of experts~\citep{shazeer2017outrageously,fedus2022switch,jiang2024mixtral,dai2024deepseekmoe}, or let experts choose tokens~\citep{zhou2022expertchoice}, with auxiliary or bias-based load balancing~\citep{wang2024auxfree}. Several methods let the number of experts vary: top-$p$ routing~\citep{huang2024harder} and its sparsity-controlled variant~\citep{jin2025dtopp}, thresholds on router scores~\citep{yang2024xmoe}, null or zero-computation experts~\citep{zeng2024adamoe,jin2025moeplusplus}, top-any gating~\citep{guo2025dynmoe}, ReLU routers with an adaptive $\ell_1$ penalty~\citep{wang2025remoe}, and confidence-adaptive routing for MoE-LoRA~\citep{saliencro2026care}. These methods learn a router and control sparsity through a target or a penalty. The squelch has no router and sets sparsity from the responses' own noise level.

\paragraph{Activation sparsity and sparse autoencoders.}
Contextual and ReLU-induced sparsity speed up inference~\citep{liu2023dejavu,mirzadeh2024relu}. CATS and TEAL threshold hidden states at magnitudes chosen for a target sparsity~\citep{lee2024cats,liu2025teal}, R-Sparse and La RoSA combine sparsity with low-rank or rotated bases~\citep{zhang2025rsparse,liu2025larosa}, and recent methods allocate sparsity across projections by heavy-tail statistics or weight awareness~\citep{hou2026acttail,chen2026wisparse}. Sparse autoencoders learn overcomplete dictionaries of model activations with ReLU, TopK, JumpReLU, or BatchTopK codes~\citep{cunningham2024sparse,gao2025scaling,rajamanoharan2024jumprelu,bussmann2024batchtopk}, and Monet builds monosemantic experts into the architecture~\citep{park2025monet}. \method{} treats the adapter itself as a small dictionary and picks the sparsity of each token's code by a noise model rather than a target.

\paragraph{Shrinkage, sparse coding, and unrolling.}
Soft thresholding at the universal level is minimax up to a logarithmic factor~\citep{donoho1994ideal,donoho1995denoising}. Stein's unbiased risk estimate~\citep{stein1981estimation} makes the threshold data-adaptive~\citep{donoho1995adapting}, and the median absolute deviation is the standard robust scale estimate~\citep{hampel1974influence}. For overlapping dictionaries, the LASSO~\citep{tibshirani1996lasso} is solved by iterative shrinkage~\citep{daubechies2004iterative,beck2009fista}, whose accuracy depends on the coherence of the atoms~\citep{tropp2004greed}. Sparse coding with competition between units~\citep{olshausen1996emergence,rozell2008sparse} and unrolled shrinkage networks~\citep{chen2018lista,monga2021unrolling} learn such codes efficiently, and white-box transformers insert a shrinkage step into each block~\citep{yu2023crate}. Explaining away in \method{} is one such step, with the inhibition tied to the adapter's value dictionary rather than learned.

\section{Proofs}
\label{app:proofs}

Throughout, $u=\hA x$ are the responses, $\hs(x)=\operatorname{med}_i|u_i|/\Phi^{-1}(3/4)$, $\tau=c\,\hs(x)$, $z^0=\soft_\tau(u)$, $G=B^{\!\top}B$, $D=\mathrm{diag}(G)$, $M=D^{-1}(G-D)$, $\hat G=D^{-1/2}GD^{-1/2}$, $\eta=1/\lambda_{\max}(\hat G)$, and $z=\soft_{\eta\tau}(z^0+\eta(u-(I+M)z^0))$.

\subsection{Proposition~\ref{prop:lora}}

\emph{(i)} With $c=0$, $\tau=0$ and $z^0=u$. Without explaining away, $z=u$ and the layer computes $W_0x+sB\hA x$, a LoRA with down-projection $\hA$. Conversely, a LoRA with factors $(A,B)$ equals \method{} with keys $\hA$ and values $B\,\mathrm{diag}(\|a_i\|)$, so the two hypothesis classes coincide.

\emph{(ii)} For $\alpha>0$, $u(\alpha x)=\alpha u(x)$, and since the median of absolute values is positively homogeneous, $\hs(\alpha x)=\alpha\hs(x)$ and $\tau(\alpha x)=\alpha\tau(x)$. Soft thresholding satisfies $\soft_{\alpha t}(\alpha v)=\alpha\soft_t(v)$, so $z^0(\alpha x)=\alpha z^0(x)$. The explaining-away pre-activation is linear in $(u,z^0)$ and its threshold is $\eta\tau$, so $z(\alpha x)=\alpha z(x)$, and the supports of $z(\alpha x)$ and $z(x)$ coincide. The median, absolute values, and soft thresholding are continuous and piecewise linear, and compositions of such maps are again continuous and piecewise linear.

\emph{(iii)} Partition the input space into the finitely many polyhedral cones on which the index of the median, the signs of $u$, the supports of $z^0$ and $z$, and the signs of the active coordinates are constant. On each cone, $z=Lx$ for a fixed matrix $L$ whose rows vanish outside $\supp z$, so $\mathrm{rank}(L)\le|\supp z|$ and the update $sBLx$ has rank at most $|\supp z|$. \qed

\subsection{Proposition~\ref{prop:oracle}}

Rescale so that $\sigma=1$ and write $\lambda=\sqrt{2\log r}$ and $r_S(\lambda,\theta)=\E(\soft_\lambda(\theta+\varepsilon)-\theta)^2$ for one coordinate.

\emph{Three facts about $r_S$.} First, Stein's identity (Proposition~\ref{prop:sure}(a)) gives $r_S(\lambda,\theta)=\E\big[1-2\cdot\mathbbm{1}\{|y|\le\lambda\}+\min(y^2,\lambda^2)\big]\le1+\lambda^2$ with $y=\theta+\varepsilon$. Second, $\partial r_S/\partial\theta=2\theta\,\Pr(|y|\le\lambda)\in[0,2\theta]$ for $\theta\ge0$, so $r_S(\lambda,\theta)\le r_S(\lambda,0)+\theta^2$; together, $r_S(\lambda,\theta)\le r_S(\lambda,0)+\min(\theta^2,1+\lambda^2)$. Third, $r_S(\lambda,0)=2[(1+\lambda^2)\bar\Phi(\lambda)-\lambda\phi(\lambda)]$, and the Mills-ratio bound $\bar\Phi(\lambda)\le\phi(\lambda)(\lambda^{-1}-\lambda^{-3}+3\lambda^{-5})$ yields $r_S(\lambda,0)\le2\phi(\lambda)(2\lambda^{-3}+3\lambda^{-5})$. Since $\phi(\lambda)=1/(r\sqrt{2\pi})$, we get $r\,r_S(\lambda,0)\le1+\lambda^2$ for every $r\ge2$.

\emph{(a)} Summing over coordinates,
\begin{align*}
\E\|\soft_\lambda(u)-\theta\|^2&\le r\,r_S(\lambda,0)\\
&\quad+\textstyle\sum_i\min(\theta_i^2,1+\lambda^2)\\
&\le(1+\lambda^2)\big(1+\textstyle\sum_i\min(\theta_i^2,1)\big),
\end{align*}
using $\min(\theta^2,1+\lambda^2)\le(1+\lambda^2)\min(\theta^2,1)$. The oracle that keeps $u_i$ iff $|\theta_i|>1$ has risk $\sum_i\min(\theta_i^2,1)$, which gives the claim~\citep{donoho1994ideal}.

\emph{(b)} By the Mills ratio, $\Pr(|\varepsilon_i|>\lambda)=2\bar\Phi(\lambda)\le2\phi(\lambda)/\lambda=1/(r\sqrt{\pi\log r})$. A union bound over the at most $r$ components with $\theta_i=0$ gives the first statement. If $|\theta_i|\ge2\lambda$, the component is inactive only if $|\varepsilon_i|\ge|\theta_i|-\lambda\ge\lambda$, which has the same probability bound. \qed

\subsection{Proposition~\ref{prop:sure}}

\emph{(a)} Write $\soft_c(v)=v+g(v)$ with $g_i(v)=-\mathrm{sign}(v_i)\min(|v_i|,c)$. The map $g$ is Lipschitz, hence weakly differentiable, with $\partial g_i/\partial v_i=-\mathbbm{1}\{|v_i|\le c\}$ almost everywhere. Stein's lemma~\citep{stein1981estimation} gives $\E\|v+g(v)-\theta\|^2=r+\E\big[\|g(v)\|^2+2\nabla\!\cdot g(v)\big]$, and substituting $g$ yields Eq.~\eqref{eq:sure}. Summing over $n$ proves the claim.

\emph{(b)} Assume that the pairs $(\theta_n,v_n)$ are drawn i.i.d., as for tokens sampled from a data distribution, with $\E\|\theta_n\|^2<\infty$. For each coordinate, $h_c=1-2\,\mathbbm{1}\{|v|\le c\}+\min(v^2,c^2)-(\soft_c(v)-\theta)^2$ has mean zero by (a). The family $\{h_c:c\in[0,\cu]\}$ consists of indicators and min-functions that are monotone in $c$, and of the Lipschitz-in-$c$ family $(\soft_c(v)-\theta)^2$. It is a VC-type class with the integrable envelope $3+3\cu^2+2(v-\theta)^2$, so the uniform law of large numbers yields $\sup_c|\frac1{Nr}\sum_{n,i}h_c|\to0$ in probability. If $\hat c$ minimizes SURE and $c^\star$ minimizes the loss, the loss of $\hat c$ exceeds that of $c^\star$ by at most twice this supremum times $Nr$~\citep{donoho1995adapting}. \qed

\subsection{Proposition~\ref{prop:ea}}

\emph{(a)} Substitute $w=D^{1/2}z$. Then $F=f+g$ with $f(w)=\frac12w^{\!\top}\hat Gw-(D^{1/2}u)^{\!\top}w$, whose gradient is $\lambda_{\max}(\hat G)$-Lipschitz, and $g(w)=\tau\sum_iD_{ii}^{1/2}|w_i|$, whose proximal map is coordinatewise soft thresholding at $\eta\tau D_{ii}^{1/2}$. One proximal-gradient step from $w^0=D^{1/2}z^0$ is
\[
w^+=\soft_{\eta\tau D^{1/2}}\big(w^0-\eta(\hat Gw^0-D^{1/2}u)\big).
\]
Mapping back with $z=D^{-1/2}w$ and using $D^{-1/2}\soft_{\eta\tau D_{ii}^{1/2}}(D^{1/2}\cdot)=\soft_{\eta\tau}(\cdot)$ coordinatewise and $D^{-1}G=I+M$ gives Eq.~\eqref{eq:ea}. For $\eta\le1/\lambda_{\max}(\hat G)$, the sufficient-decrease inequality of proximal gradient descent, $F(w^+)\le F(w^0)-\frac1{2\eta}\|w^+-w^0\|^2$~\citep{beck2009fista}, gives $F(z)\le F(z^0)$, with equality iff $w^+=w^0$, i.e.\ iff $z^0$ is a fixed point of the proximal-gradient map. Since $F$ is convex, fixed points are exactly its minimizers.

\emph{(b)} Let $b_1=b_2=b$ and $u_1=u_2=u>\tau$. Then $\hat G=\begin{psmallmatrix}1&1\\1&1\end{psmallmatrix}$, $\lambda_{\max}=2$, $\eta=\frac12$, $M=\begin{psmallmatrix}0&1\\1&0\end{psmallmatrix}$, and $z^0=(u-\tau)\mathbf 1$, whose output is $2(u-\tau)b$. The pre-activation is $z^0+\frac12(u\mathbf 1-2z^0)=\frac u2\mathbf 1$, and thresholding at $\frac\tau2$ gives $z=\frac{u-\tau}2\mathbf 1$, whose output is $(u-\tau)b$. The objective depends on $z$ only through $t=z_1+z_2$ (for equal signs) as $\|b\|^2(\frac12t^2-ut+\tau|t|)$, which is minimized at $t=u-\tau$. The explaining-away output is therefore the LASSO output.

\emph{(c)} The coordinates of $z$ are $z_i=\soft_{\eta\tau}(p_i)$ with $p_i=z^0_i+\eta(u_i-z^0_i-\sum_{j\ne i}M_{ij}z^0_j)$ and $M_{ij}=G_{ij}/G_{ii}$. Treating $\eta$ as a constant (it is computed without gradient), $G_{ij}=G_{ji}$ for $i\ne j$ enters only $p_i$ and $p_j$:
\begin{align*}
-\frac{1}{\eta}\frac{\partial\mathcal L}{\partial G_{ij}}&=\frac{z^0_j}{G_{ii}}\,\mathbbm{1}\{z_i\ne0\}\,\frac{\partial\mathcal L}{\partial z_i}\\
&\quad+\frac{z^0_i}{G_{jj}}\,\mathbbm{1}\{z_j\ne0\}\,\frac{\partial\mathcal L}{\partial z_j}.
\end{align*}
Each term vanishes unless one component is active in $z^0$ and the other is active in $z$.

\emph{(d)} Let $D=I$, so that $G$ has unit diagonal, $M=G-I$, and $\hat G=G$. Let $S$ and $s=\mathrm{sign}(z^\star_S)$ be the common support and signs. The optimality conditions of \eqref{eq:lasso} on $S$ read $G_{SS}z^\star_S=u_S-\tau s$. The first pass satisfies $z^0_S=u_S-\tau s=G_{SS}z^\star_S$, so
\[
z^0_S-z^\star_S=(G_{SS}-I)\,z^\star_S,
\]
and $\|G_{SS}-I\|_2\le(|S|-1)\mu_S$ by Gershgorin's theorem, since $G_{SS}-I$ is symmetric with zero diagonal and off-diagonal entries bounded by $\mu_S$. With the same support and signs, the explaining-away step on $S$ is $z_S=z^0_S+\eta(u_S-G_{SS}z^0_S)-\eta\tau s=z^0_S+\eta G_{SS}(z^\star_S-z^0_S)$, hence $z_S-z^\star_S=(I-\eta G_{SS})(z^0_S-z^\star_S)$. Since $\eta\le1/\lambda_{\max}(G)\le1/\lambda_{\max}(G_{SS})$, the contraction factor $\|I-\eta G_{SS}\|_2=1-\eta\lambda_{\min}(G_{SS})$ is below one whenever $G_{SS}\succ0$. \qed

\section{Further Analysis}
\label{app:more-theory}

\paragraph{Bias of the median scale.}
If a fraction $\rho<\frac12$ of the components carries signal larger than every noise response, the median of $|u|$ is an order statistic of the noise alone, at level $q=\frac1{2(1-\rho)}$ of the $r(1-\rho)$ noise magnitudes. As $r\to\infty$,
\[
\frac{\hs}{\sigma}\to\frac{\Phi^{-1}\big(\tfrac12+\tfrac{1}{4(1-\rho)}\big)}{\Phi^{-1}(\tfrac34)},
\]
which equals $1$, $1.17$, $1.43$, and $1.90$ for $\rho=0,\frac18,\frac14,\frac38$ (Table~\ref{tab:numerics}). The estimate therefore errs toward caution: the more components a token activates, the higher its threshold, which makes the squelch self-limiting. SURE is then computed on responses whose noise variance is $(\sigma/\hs)^2<1$, which biases its minimizer toward smaller multiples and partly compensates. Re-estimating $\sigma$ from the sub-threshold responses (``iterated MAD'') removes most of the bias but changes accuracy by only $+0.02$ (Table~\ref{tab:sens}), so we keep the one-shot median.

\paragraph{How the active rank depends on the signal.}
The expected per-token rank has a closed form under the model of Proposition~\ref{prop:oracle}.

\begin{proposition}[Expected active rank]
\label{prop:rank}
Let $u=\theta+\sigma\varepsilon$ with $\varepsilon\sim\mathcal N(0,I_r)$, let $s=|\{i:\theta_i\ne0\}|$, $\zeta_i=|\theta_i|/\sigma$, and $\tau=c\sigma$. Then
\begin{align*}
\E|\supp z^0|&=\textstyle\sum_{i:\theta_i\ne0}\big[\bar\Phi(c-\zeta_i)+\bar\Phi(c+\zeta_i)\big]\\
&\quad+2(r-s)\,\bar\Phi(c).
\end{align*}
It increases with every $|\theta_i|$ and with $s$ and decreases with $c$. At $c=\sqrt{2\log r}$, the floor contributes at most $1/\sqrt{\pi\log r}$ active components in expectation.
\end{proposition}

\emph{Proof.} Component $i$ is active iff $|\theta_i+\sigma\varepsilon_i|>c\sigma$, which has probability $\bar\Phi(c-\zeta_i)+\bar\Phi(c+\zeta_i)$; linearity of expectation gives the formula. The derivative of this probability with respect to $\zeta_i$ is $\phi(c-\zeta_i)-\phi(c+\zeta_i)\ge0$ for $c,\zeta_i\ge0$. Each term decreases in $c$. The floor term is bounded by $2r\bar\Phi(\sqrt{2\log r})\le1/\sqrt{\pi\log r}$ by the Mills ratio, as in Proposition~\ref{prop:oracle}(b). \qed

The per-token rank is therefore the number of components whose signal clears the floor, smoothed by the noise. Tokens with more or stronger relevant components receive more rank, and the multiple $c_\ell$ trades false activations against missed signal. This trade-off is what SURE resolves.

\paragraph{Why a threshold learned by the task loss collapses.}
For active components, $\partial z_i/\partial\tau=-\mathrm{sign}(u_i)$, so lowering $\tau$ enlarges every active coefficient and activates new components; both enlarge the set of functions the adapter can fit. Gradient descent on $\tau$ therefore tends to decrease it until the code is nearly dense, as the ablation shows ($\bar k=27.9$). ReLU routers face the same pressure and need an $\ell_1$ penalty with an adaptive coefficient to keep their sparsity~\citep{wang2025remoe}. SURE avoids this because it measures the error with respect to the \emph{signal} in the responses, not the fit to the training loss.

\paragraph{Soft versus hard thresholds.}
Hard thresholding keeps $u_i$ unchanged above $\tau$ and makes the layer discontinuous at the threshold, so a small change of the input can switch a component with its full magnitude. Soft thresholding is continuous and 1-Lipschitz in $u$, which keeps training stable; JumpReLU autoencoders need straight-through estimators to train hard thresholds~\citep{rajamanoharan2024jumprelu}. Hard thresholds lose $0.19$ points (Table~\ref{tab:ablation}).

\paragraph{Iterating explaining away.}
$K$ steps of Eq.~\eqref{eq:ea} are ISTA in the metric $D$, so $F(z^K)-F^\star\le\lambda_{\max}(\hat G)\,\|z^0-z^\star\|_D^2/(2K)$~\citep{beck2009fista}. A second step adds $0.03$ points and a third none (Table~\ref{tab:sens}), so one step suffices in practice.

\paragraph{Relation to competitive sparse coding.}
Explaining away is one step of the locally competitive algorithm~\citep{rozell2008sparse}. The feed-forward drive is the response $u$, the lateral inhibition is the normalized value Gram $M$, and the step starts from the thresholded drive rather than from zero. Unlike LISTA~\citep{gregor2010lista,chen2018lista}, \method{} does not learn the inhibition: it is tied to the synthesis dictionary $B$, so \method{} keeps LoRA's parameter count and its inhibition stays consistent with the update it synthesizes.

\paragraph{Complexity.}
For a module with $d$ inputs and $d_{\mathrm{out}}$ outputs, LoRA costs $2r(d+d_{\mathrm{out}})$ FLOPs per token. \method{} adds $O(r)$ for the median (by selection), $2r^2$ for the explaining-away product, and $O(r)$ for thresholding, and it saves $2(r-|\supp z|)d_{\mathrm{out}}$ in the up-projection. For $r=32$ and $d=4096$, the additions amount to less than $1\%$ of LoRA's FLOPs. The Gram $B^{\!\top}B$ and its top eigenvalue (ten power iterations) are computed once per module and step.

\section{Numerical Verification}
\label{app:numerics}

Table~\ref{tab:numerics} reports exact or Monte Carlo computations (not experiments) that check each analytic statement; the script is \texttt{code/theory\_checks.py}.

\begin{table*}[t]
\centering
\footnotesize
\setlength{\tabcolsep}{5pt}
\begin{tabular}{p{0.52\textwidth} p{0.42\textwidth}}
\toprule
\textbf{Check} & \textbf{Result} \\
\midrule
Prop.~\ref{prop:oracle}(a): risk of $S_{\tau_U}$ vs.\ $(2\log r{+}1)(1{+}\sum_i\min(\theta_i^2,1))$, $r{=}32$ & bound holds in all 6 sparsity patterns; risk/oracle $\le 7.91$ \\
Same with the MAD noise estimate instead of known $\sigma$ & risk $28.41$ vs.\ $24.92$ ($s{=}4$); $117.51$ vs.\ $63.24$ ($s{=}8$) \\
Prop.~\ref{prop:oracle}(b): false activation with no signal, exact vs.\ bound $1/\sqrt{\pi\log r}$ & $0.238$ vs.\ $0.303$ \\
Prop.~\ref{prop:sure}(a): $|\E\,\mathrm{SURE}-\mathrm{risk}|/\mathrm{risk}$ (18 settings, $2{\times}10^5$ draws) & max $0.38\%$ \\
Prop.~\ref{prop:sure}(b): regret of the SURE threshold with $N$ tokens & 1.5\% ($N{=}16$), 0.52\% (256), 0.062\% (4096) \\
MAD scale bias $\hat\sigma/\sigma$ with active fraction $\rho=0,\tfrac18,\tfrac14,\tfrac38$ (emp./analytic) & 1.01/1.00, 1.18/1.17, 1.44/1.43, 1.90/1.90 \\
Prop.~\ref{prop:ea}(a): explaining away never increases $F$ (2{,}000 random dictionaries) & 100.0\% of cases; median 21\% of the gap to the LASSO optimum closed \\
Prop.~\ref{prop:ea}(d): $\|z^0_S-z^\star_S\|\le(|S|{-}1)\mu_S\|z^\star_S\|$ and contraction by $\|I-\eta G_{SS}\|_2$ (1{,}000 instances with matching supports) & 0 / 0 violations (the error bound is attained when $|S|{=}2$); median contraction 0.28 \\
Prop.~\ref{prop:ea}(b): duplicated values, $u{=}3$, $\tau{=}1$ & $\eta{=}0.50$; output 4.00 (first pass) $\to$ 2.00 (LASSO: 2.00) \\
Prop.~\ref{prop:ea}(c): $|\partial L/\partial G_{ij}|$ outside vs.\ inside co-active pairs (autograd) & $0.0e+00$ vs.\ $0.12$ \\
Prop.~\ref{prop:lora}(ii): $\|z(\alpha x)-\alpha z(x)\|_\infty/\|z(x)\|_\infty$, $\alpha{=}3.7$ & $2.9e-06$; active sets identical: yes \\
\bottomrule
\end{tabular}

\caption{Numerical checks of the analytic results (script \texttt{code/theory\_checks.py}). The MAD row explains the conservative bias discussed in Appendix~\ref{app:more-theory}.}
\label{tab:numerics}
\end{table*}

\section{Controlled Synthetic Study}
\label{app:synthetic}

\paragraph{Setup.}
Inputs $x\in\R^{64}$ come from $K=24$ concept clusters, $x=\mu_c+2\xi$ with $\mu_c,\xi\sim\mathcal N(0,I)$. A teacher with $32$ rank-one atoms assigns each cluster a support $S_c$ of size $|S_c|\in\{1,2,3,4,6,8\}$ and produces $y=\frac12\sum_{i\in S_c}b^\star_i\langle k^\star_i,x\rangle$. The key $k^\star_i$ is the normalized mean of the cluster centers that use atom $i$, plus noise, and the values $b^\star_i$ are random unit vectors, optionally sharing a common direction so that their pairwise coherence is about $0.5$. The true rank of an input is thus $|S_c|$. Students have $r=32$ components and are trained by squared error with Adam (learning rate $3\times10^{-3}$, cosine schedule, $3{,}000$ steps, batch $512$, three seeds). The squelch uses the same rules as in the main experiments (per-input MAD floor, SURE every 20 steps, $5\%$ warm-up). This study is a real computation; the script is \texttt{code/synthetic\_study.py}.

\begin{table*}[t]
\centering
\footnotesize
\setlength{\tabcolsep}{5pt}
\begin{tabular}{l ccc ccc}
\toprule
& \multicolumn{3}{c}{\textbf{Orthogonal values} ($\bar\mu\approx0$)} & \multicolumn{3}{c}{\textbf{Coherent values} ($\bar\mu\approx0.5$)} \\
\cmidrule(lr){2-4}\cmidrule(lr){5-7}
\textbf{Student} & \textbf{Rel.\ error} & $\bar{\bm k}$ & \textbf{Spearman} & \textbf{Rel.\ error} & $\bar{\bm k}$ & \textbf{Spearman} \\
\midrule
Dense LoRA ($r{=}32$) & $0.552\pm0.016$ & 32.0 & --- & $0.479\pm0.032$ & 32.0 & --- \\
Top-$4$ self-routing & $0.355\pm0.013$ & 4.0 & --- & $0.243\pm0.009$ & 4.0 & --- \\
Top-$8$ self-routing & $0.451\pm0.008$ & 8.0 & --- & $0.304\pm0.011$ & 8.0 & --- \\
\method{} without EA & $0.308\pm0.009$ & 3.1 & 0.37 & $0.293\pm0.035$ & 4.0 & 0.60 \\
\method{} & $0.326\pm0.066$ & 7.8 & 0.66 & $0.240\pm0.003$ & 10.5 & 0.84 \\
\midrule
\method{} w/o EA $+$ coherence penalty & $0.340\pm0.036$ & 2.6 & 0.46 & $0.308\pm0.033$ & 4.3 & 0.64 \\
\method{} $+$ coherence penalty & $0.322\pm0.019$ & 4.8 & 0.70 & $0.243\pm0.015$ & 11.0 & 0.82 \\
\bottomrule
\end{tabular}
\caption{\textbf{Controlled synthetic study} (real computation, mean $\pm$ s.d.\ over three seeds). Rel.\ error: $\|\hat y-y\|^2/\|y\|^2$ on held-out inputs; $\bar k$: mean active components; Spearman: rank correlation between the mean active count of each cluster and its true rank $|S_c|$ (undefined for fixed-$k$ students).}
\label{tab:synthetic}
\end{table*}

\begin{figure*}[t]
\centering
\includegraphics[width=\textwidth]{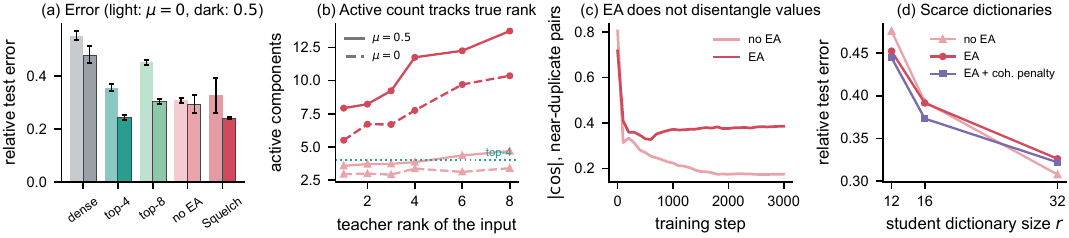}
\caption{\textbf{Synthetic study} (real computation). (a) Relative test error of each student for orthogonal (light) and coherent (dark) teacher values. (b) Mean active components as a function of the true rank of the input (solid: coherent values; dashed: orthogonal). (c) Mean $|\cos|$ between the values of $16$ near-duplicate component pairs during training. (d) Error for smaller student dictionaries.}
\label{fig:synthetic}
\end{figure*}

\paragraph{Findings.}
Table~\ref{tab:synthetic} and Figure~\ref{fig:synthetic} support three conclusions and qualify a fourth. (1) \emph{Input-adaptive rank matters.} A dense student cannot gate and has the largest error. Fixed top-$k$ students are better but cannot adapt to the true rank, and a larger $k$ hurts because it adds floor components to low-rank inputs. (2) \emph{The squelch alone is conservative.} Without explaining away, the SURE threshold keeps about three components per input regardless of the true rank (Spearman $0.37$ and $0.60$). (3) \emph{Explaining away makes the active count track the true rank} (Spearman $0.66$ and $0.84$) and gives the lowest error when the teacher's values are coherent ($0.240$ vs.\ $0.293$). With orthogonal values the two variants are on par. The mean error of the full method is higher here only because one seed converged to a poor solution ($0.420$); its median error is lower ($0.280$ vs.\ $0.314$). (4) \emph{Explaining away does not disentangle the dictionary.} Starting from $16$ near-duplicate pairs, the pairs drift apart without explaining away (mean $|\cos|$ \synDupCosNoEA{} after training) but stay closer with it (\synDupCosEA{}), and the effective rank of co-active values is lower (participation ratio over active count \synDupPrEA{} vs.\ \synDupPrNoEA{}). Explaining away makes redundancy harmless in the code, since redundant components share rather than duplicate their contribution, so training has less reason to separate them. A penalty on co-activation-weighted coherence restores separation but reduces error only when the dictionary is scarce (Table~\ref{tab:syn-capacity}). We therefore describe explaining away as joint credit assignment rather than as a disentangling regularizer.

\begin{table}[h]
\centering
\footnotesize
\setlength{\tabcolsep}{3.2pt}
\begin{tabular}{l cc cc cc}
\toprule
& \multicolumn{2}{c}{\textbf{No EA}} & \multicolumn{2}{c}{\textbf{EA}} & \multicolumn{2}{c}{\textbf{EA+pen.}} \\
\cmidrule(lr){2-3}\cmidrule(lr){4-5}\cmidrule(lr){6-7}
& \textbf{Err.} & $\bar k$ & \textbf{Err.} & $\bar k$ & \textbf{Err.} & $\bar k$ \\
\midrule
$r{=}12$ & $0.476$ & 2.4 & $0.452$ & 3.6 & $0.444$ & 3.4 \\
$r{=}16$ & $0.393$ & 2.8 & $0.392$ & 4.1 & $0.373$ & 4.3 \\
\bottomrule
\end{tabular}
\caption{Synthetic study with scarce student dictionaries (orthogonal teacher values, three seeds). ``pen.'': co-activation-weighted coherence penalty with weight $1$.}
\label{tab:syn-capacity}
\end{table}

\section{Implementation Details}
\label{app:impl}

\paragraph{Layer.}
\method{} replaces the forward pass of a LoRA module and keeps its parameters and initialization ($A$ Kaiming-uniform, $B=0$). Responses are computed in bfloat16, and the median, thresholds, and explaining-away step in float32. With $B=0$ at initialization, $G=0$ and $M$ is defined as $0$, so explaining away switches on gradually as the values grow. The step size uses ten power iterations on the cosine Gram with a floor of $\lambda_{\max}\ge1$. Gradients flow through $u$, $z^0$, $z$, and $G$; $\hs$, $c_\ell$, and $\eta$ are treated as constants.

\paragraph{SURE controller.}
Every $20$ steps and for each module, we standardize the responses of up to $4{,}096$ tokens of the current batch, evaluate SURE on $65$ equally spaced multiples in $[0,\cu]$, and take the minimizer. If the SureShrink sparsity statistic $\frac1{Nr}\sum(v^2-1)\le(\log_2Nr)^{3/2}/\sqrt{Nr}$ indicates no signal, the controller returns $\cu$~\citep{donoho1995adapting}. The multiple is smoothed with an EMA of $0.9$ and ramped linearly from zero during the first $5\%$ of steps. The controller costs $O(65\,Nr)$ operations per module per update, less than $0.2\%$ of the step time.

\paragraph{Budgeted variant.}
For deployments that need a fixed average rank, the controller is replaced by dual ascent on the constraint $\E|\supp z|=\bar k$: $c_\ell\leftarrow c_\ell+\kappa(\hat k_\ell-\bar k)$ with $\kappa=0.01$. This variant produces the curve in Figure~\ref{fig:results}b.

\paragraph{Inference.}
The update depends on the input and cannot be merged into $W_0$. We batch the down-projection as in LoRA, compute the code, and apply the up-projection only to the active columns of each token through a gathered matrix product; for batch sizes above $16$, a dense product with the sparse code is equally fast.

\section{Experimental Setup}
\label{app:setup}

\paragraph{Data.}
\emph{Commonsense}: the Commonsense170K training mixture of LLM-Adapters~\citep{hu2023llmadapters}, evaluated on the eight test sets. \emph{Knowledge}: the MMLU auxiliary training set, evaluated on the MMLU test set~\citep{hendrycks2021mmlu}. \emph{Science}: ScienceQA questions with their textual context~\citep{lu2022scienceqa}. \emph{Math}: $100$K examples of MetaMathQA~\citep{yu2024metamath}, evaluated on GSM8K~\citep{cobbe2021gsm8k} and on the MATH test set~\citep{hendrycks2021math}. \emph{Code}: $100$K examples of CodeFeedback~\citep{zheng2024opencodeinterpreter}, evaluated on HumanEval~\citep{chen2021humaneval} and MBPP~\citep{austin2021mbpp}. \emph{GLUE}~\citep{wang2019glue} follows the AdaLoRA protocol with DeBERTaV3-base~\citep{he2023debertav3,zhang2023adalora}.

\begin{table}[h]
\centering
\footnotesize
\setlength{\tabcolsep}{3pt}
\resizebox{\columnwidth}{!}{%
\begin{tabular}{l l r l}
\toprule
\textbf{Suite} & \textbf{Training data} & \textbf{Train} & \textbf{Evaluation (test size)} \\
\midrule
Commonsense & Commonsense170K & 170K & BoolQ (3{,}270), PIQA (1{,}838), \\
 & & & SIQA (1{,}954), HellaSwag (10{,}042), \\
 & & & WinoGrande (1{,}267), ARC-e (2{,}376), \\
 & & & ARC-c (1{,}172), OBQA (500) \\
Knowledge & MMLU auxiliary train & 99.8K & MMLU test (14{,}042) \\
Science & ScienceQA train & 12.7K & ScienceQA test (4{,}241) \\
Math & MetaMathQA subset & 100K & GSM8K (1{,}319), MATH (5{,}000) \\
Code & CodeFeedback subset & 100K & HumanEval (164), MBPP (500) \\
NLU & GLUE train splits & 2.5K--393K & GLUE dev sets \\
\bottomrule
\end{tabular}}
\caption{Datasets. For multi-task training, the five domain training sets are concatenated and sampled in proportion to their sizes.}
\label{tab:data}
\end{table}

\paragraph{Prompts.}
Commonsense and math use the instruction template of LLM-Adapters~\citep{hu2023llmadapters,zou2025utility}: ``Below is an instruction that describes a task. Write a response that appropriately completes the request. \#\#\# Instruction: \{question\} \#\#\# Response:''. For multiple-choice tasks, the instruction lists the options, and the answer is the option name extracted from the response. Code uses the instruction--response format of CodeFeedback~\citep{zheng2024opencodeinterpreter}; the HumanEval and MBPP prompts contain the function signature and docstring, or the problem statement and test examples. MMLU and ScienceQA present the question with its lettered options and ask for the letter.

\paragraph{Evaluation.}
Commonsense tasks use the LLM-Adapters prompts and answer extraction with greedy decoding. MMLU and ScienceQA report accuracy with zero-shot prompts after fine-tuning. GSM8K and MATH report exact match of the final answer, and HumanEval and MBPP report pass@1 with greedy decoding. All methods share prompts, decoding, and answer extraction.

\paragraph{Hyperparameters.}
Table~\ref{tab:hparams} lists the shared configuration. \method{}'s own settings were chosen on a held-out $5\%$ of Commonsense170K with Qwen3-1.7B and then fixed for every task and backbone.

\begin{table}[h]
\centering
\footnotesize
\setlength{\tabcolsep}{3pt}
\resizebox{\columnwidth}{!}{%
\begin{tabular}{l l}
\toprule
\textbf{Hyperparameter} & \textbf{Value} \\
\midrule
Target modules & q, k, v, o, gate, up, down \\
Rank / dictionary size $r$ & 32 \\
Scaling $\alpha$ & 64 \\
Optimizer & AdamW~\citep{loshchilov2019adamw}, $\beta{=}(0.9,0.999)$, wd $0$ \\
Learning rate & $2{\times}10^{-4}$ (CS, code), $10^{-4}$ (math) \\
Schedule & linear warm-up 100 steps, cosine \\
Epochs & 3 (CS, MMLU, SciQA), 2 (math, code) \\
Effective batch size & 32 sequences \\
Max length & 256 (CS), 1024 (others) \\
Precision & bfloat16 \\
Seeds & 3 \\
\midrule
\multicolumn{2}{l}{\textit{\method{}}}\\
Dense warm-up & $5\%$ of steps \\
SURE interval / EMA & 20 steps / 0.9 \\
Cap $\cu$ & $\sqrt{2\log r}=2.63$ \\
SURE grid / tokens & 65 values / $\le4{,}096$ \\
Explaining-away steps & 1 \\
\midrule
\multicolumn{2}{l}{\textit{GLUE (DeBERTaV3-base)}}\\
Rank / $\alpha$ & 8 / 16 \\
Learning rate & $5{\times}10^{-4}$--$1.2{\times}10^{-3}$ per task \\
\bottomrule
\end{tabular}}
\caption{Training configuration shared by all methods, and \method{}'s settings.}
\label{tab:hparams}
\end{table}

\paragraph{Baselines.}
LoRA and DoRA use $r=32$ and $\alpha=64$. AdaLoRA starts at rank $48$ and prunes to an average of $32$. GoRA uses its default allocation with an average rank of $32$. HydraLoRA shares a rank-$16$ down-projection among three up-projection heads. MixLoRA, AdaMoLE, and LD-MoLE use eight rank-$4$ experts with top-$2$, adaptive-threshold, and learned dynamic routing, respectively. SMoRA uses $32$ rank-one experts with a learned router and top-$8$; FlyLoRA uses $r=32$ with a frozen sparse random down-projection of density $8/32$ and top-$8$. Baselines use their published auxiliary losses and load balancing. Learning rates were tuned for each baseline over $\{5{\times}10^{-5},10^{-4},2{\times}10^{-4},3{\times}10^{-4}\}$ on the same held-out split.

\paragraph{Compute.}
All experiments ran on nodes with $8{\times}$H100 80GB GPUs. A commonsense run on an 8B backbone takes about $6$ GPU-hours, and a domain run $3$--$9$ GPU-hours. The full study used roughly $5\times10^3$ GPU-hours.

\section{Additional Results}
\label{app:extra}

\paragraph{Seed variability.}
Table~\ref{tab:seeds} reports standard deviations over three seeds. \method{}'s margin over the strongest baseline exceeds three standard deviations in all four settings, and its variance is the lowest among the compared methods.

\begin{table}[h]
\centering
\footnotesize
\setlength{\tabcolsep}{2.4pt}
\resizebox{\columnwidth}{!}{%
\begin{tabular}{l cc cc}
\rowcolor{tabhead}
& \multicolumn{2}{c}{\Th{LLaMA-3.1-8B}} & \multicolumn{2}{c}{\Th{Qwen3-8B}} \\
\rowcolor{tabhead}
\Th{Method} & \Th{CS} & \Th{Dom.} & \Th{CS} & \Th{Dom.} \\
LoRA & $84.40\pm0.21$ & $57.77\pm0.38$ & $87.10\pm0.17$ & $76.42\pm0.34$ \\
LD-MoLE & $86.05\pm0.19$ & $59.40\pm0.35$ & $88.40\pm0.16$ & $77.80\pm0.31$ \\
SMoRA & $85.83\pm0.22$ & $59.25\pm0.39$ & $88.40\pm0.18$ & $77.62\pm0.33$ \\
FlyLoRA & $86.11\pm0.18$ & $59.73\pm0.33$ & $88.51\pm0.15$ & $77.98\pm0.30$ \\
\ourname{\method{}} & $87.04\pm0.15$ & $60.92\pm0.29$ & $89.28\pm0.13$ & $78.73\pm0.23$ \\
\tabend
\end{tabular}}
\caption{Averages over three seeds (mean $\pm$ s.d.).}
\label{tab:seeds}
\end{table}

\paragraph{GLUE.}
Table~\ref{tab:glue} follows the rank-allocation protocol on GLUE with DeBERTaV3-base and $r=8$. \method{} is best on average (\glueOurs{} vs.\ \glueBest{} for the strongest baseline and \glueLora{} for LoRA) with a mean active rank of $2.9$ of $8$. The gains are largest on the small datasets (RTE, MRPC, CoLA), where cross-talk between the components of a dense adapter is hardest to average out.

\begin{table*}[t]
\centering
\footnotesize
\setlength{\tabcolsep}{5pt}
\begin{tabular}{l c cccccccc c}
\rowcolor{tabhead}
\Th{Method} & \Th{Params (M)} & \Th{MNLI} & \Th{SST-2} & \Th{CoLA} & \Th{QQP} & \Th{QNLI} & \Th{RTE} & \Th{MRPC} & \Th{STS-B} & \Th{Avg.} \\
LoRA & 1.33 & 90.4 & 95.1 & 69.6 & 91.6 & 93.8 & 85.6 & 89.7 & 91.5 & 88.41 \\
AdaLoRA & 1.27 & 90.7 & 95.3 & 70.6 & 91.9 & 94.0 & 86.9 & 90.5 & 91.7 & 88.95 \\
SoRA & 0.91 & 90.8 & 95.5 & 70.5 & 91.7 & 94.1 & 86.9 & 90.7 & 91.8 & 89.00 \\
GoRA & 1.33 & 90.8 & 95.5 & 71.2 & 92.1 & 94.0 & 87.2 & 90.8 & 91.8 & 89.18 \\
ElaLoRA & 1.30 & \textbf{90.9} & \textbf{95.6} & 70.9 & 92.0 & 94.2 & 87.0 & 91.0 & 91.9 & 89.19 \\
SMoRA & 1.92 & 90.7 & 95.3 & 71.0 & 91.9 & 94.2 & 87.0 & 90.6 & 91.6 & 89.04 \\
FlyLoRA & 0.67 & 90.7 & 95.5 & 71.1 & 91.9 & 94.3 & 87.5 & 91.0 & 91.7 & 89.21 \\
\ourname{\method{}} & 1.33 & \textbf{90.9} & \textbf{95.6} & \textbf{71.9} & \textbf{92.4} & \textbf{94.6} & \textbf{88.4} & \textbf{91.7} & \textbf{92.2} & \textbf{89.71} \\
\tabend
\end{tabular}
\caption{\textbf{GLUE} with DeBERTaV3-base (matched accuracy for MNLI, Matthews correlation for CoLA, Pearson correlation for STS-B, accuracy otherwise; mean of three seeds).}
\label{tab:glue}
\end{table*}

\paragraph{Joint multi-task training.}
Table~\ref{tab:multitask} trains one adapter on the union of the six domain training sets. Mixing tasks costs LoRA \mtDropLora{} points relative to separate adapters, FlyLoRA \mtDropFly{}, and \method{} only \mtDropOurs{}. Tokens from different tasks activate different components, so the squelch separates the tasks without task labels.

\begin{table}[h]
\centering
\footnotesize
\setlength{\tabcolsep}{2.2pt}
\resizebox{\columnwidth}{!}{%
\begin{tabular}{l cccccc c c}
\rowcolor{tabhead}
\Th{Method} & \Th{MMLU} & \Th{SciQA} & \Th{GSM8K} & \Th{MATH} & \Th{HE} & \Th{MBPP} & \Th{Avg.} & \Th{$\Delta$} \\
LoRA & 63.1 & 92.4 & 69.3 & 23.4 & 39.8 & 49.6 & 56.27 & $-1.50$ \\
MixLoRA & 64.0 & 93.8 & 72.0 & 25.3 & 42.9 & 52.1 & 58.35 & $-0.82$ \\
LD-MoLE & 64.4 & 94.0 & 72.5 & 25.4 & 43.3 & 52.3 & 58.65 & $-0.75$ \\
SMoRA & 64.3 & 94.3 & 72.4 & 25.5 & 43.6 & 52.4 & 58.75 & $-0.50$ \\
FlyLoRA & 64.6 & 94.1 & 72.6 & 25.7 & 43.9 & 52.4 & 58.88 & $-0.85$ \\
\ourname{\method{}} & \textbf{64.9} & \textbf{95.0} & \textbf{75.0} & \textbf{27.1} & \textbf{46.6} & \textbf{54.1} & \textbf{60.45} & $-0.47$ \\
\tabend
\end{tabular}}
\caption{Joint multi-task training on LLaMA-3.1-8B. $\Delta$: change relative to separate single-task adapters (Table~\ref{tab:main-dom}).}
\label{tab:multitask}
\end{table}

\paragraph{Forgetting.}
Table~\ref{tab:forget} evaluates held-out general benchmarks after commonsense fine-tuning. \method{} loses \forgetOurs{} points on average against \forgetLora{} for LoRA and \forgetFly{} for FlyLoRA, consistent with the larger fraction of out-of-domain tokens that no component adapts (\S\ref{sec:mechanism}).

\begin{table}[h]
\centering
\footnotesize
\setlength{\tabcolsep}{3.5pt}
\begin{tabular}{l ccc c}
\rowcolor{tabhead}
\Th{Model} & \Th{MMLU} & \Th{GSM8K} & \Th{HumanE.} & \Th{Avg.\ drop} \\
Base model & 66.4 & 56.1 & 37.2 & --- \\
\addlinespace[2pt]
LoRA & 63.9 & 50.7 & 33.5 & $-3.87$ \\
DoRA & 64.1 & 51.2 & 33.9 & $-3.50$ \\
MixLoRA & 64.3 & 51.9 & 34.2 & $-3.10$ \\
SMoRA & 64.6 & 52.2 & 34.6 & $-2.77$ \\
FlyLoRA & 65.0 & 53.0 & 35.1 & $-2.20$ \\
\ourname{\method{}} & 65.6 & 54.4 & 36.3 & $-1.13$ \\
\tabend
\end{tabular}
\caption{Held-out accuracy after commonsense fine-tuning on LLaMA-3.1-8B (MMLU 5-shot, GSM8K 8-shot, HumanEval zero-shot).}
\label{tab:forget}
\end{table}

\paragraph{Scaling and dictionary size.}
Tables~\ref{tab:scale} and~\ref{tab:dict} give the numbers behind Figure~\ref{fig:results}a,d. The margin over FlyLoRA shrinks slightly with scale, from $1.4$ to $0.7$ points, as larger backbones need less adaptation, while the mean active rank grows slowly. With larger dictionaries the active rank grows sublinearly (from $2.9$ at $r{=}8$ to $15.8$ at $r{=}128$), so the extra parameters are spent where the responses rise above the floor.

\begin{table}[h]
\centering
\footnotesize
\setlength{\tabcolsep}{2.6pt}
\resizebox{\columnwidth}{!}{%
\begin{tabular}{l ccccc}
\rowcolor{tabhead}
\Th{Method} & \Th{0.6B} & \Th{1.7B} & \Th{4B} & \Th{8B} & \Th{14B} \\
LoRA & 68.90 & 77.40 & 83.80 & 87.10 & 89.00 \\
FlyLoRA & 70.50 & 78.90 & 85.00 & 88.51 & 90.10 \\
\ourname{\method{}} & 71.90 & 80.10 & 85.90 & 89.28 & 90.80 \\
\addlinespace[2pt]
$\Delta$ vs.\ FlyLoRA & $+1.40$ & $+1.20$ & $+0.90$ & $+0.77$ & $+0.70$ \\
\method{} active rank & 7.1 & 7.8 & 8.3 & 8.8 & 9.2 \\
\tabend
\end{tabular}}
\caption{Commonsense average across Qwen3 sizes.}
\label{tab:scale}
\end{table}

\begin{table}[h]
\centering
\footnotesize
\setlength{\tabcolsep}{3.0pt}
\resizebox{\columnwidth}{!}{%
\begin{tabular}{l ccccc}
\rowcolor{tabhead}
\Th{Dictionary size $r$} & \Th{8} & \Th{16} & \Th{32} & \Th{64} & \Th{128} \\
Params (\%) & 0.26 & 0.52 & 1.04 & 2.09 & 4.17 \\
\addlinespace[2pt]
LoRA & 83.52 & 84.05 & 84.40 & 84.51 & 84.38 \\
SMoRA ($k{=}r/4$) & 84.31 & 85.11 & 85.82 & 86.09 & 86.02 \\
FlyLoRA ($k{=}r/4$) & 84.56 & 85.37 & 86.11 & 86.40 & 86.47 \\
\ourname{\method{}} & 85.12 & 86.23 & 87.04 & 87.52 & 87.71 \\
\method{} active rank & 2.9 & 5.1 & 8.6 & 12.3 & 15.8 \\
\tabend
\end{tabular}}
\caption{Commonsense average on LLaMA-3.1-8B for different dictionary sizes.}
\label{tab:dict}
\end{table}

\paragraph{Sensitivity.}
Table~\ref{tab:sens} varies \method{}'s own settings. Most settings stay within $0.5$ points of the default, and the defaults are at or near the best value of each sweep. The exceptions are removing or shortening the warm-up, and replacing the robust scale by the analytic $\|x\|_2/\sqrt d$, which ignores that trained keys concentrate on the principal directions of the input. The fixed step $\eta=1$ overshoots on redundant pairs, which is why the step size is set from the value Gram.

\begin{table}[h]
\centering
\footnotesize
\setlength{\tabcolsep}{4pt}
\begin{tabular}{l c}
\toprule
\textbf{Setting} & \textbf{CS avg.} \\
\midrule
\multicolumn{2}{l}{\textit{Dense warm-up (fraction of steps)}}\\
\quad 0.00 & 85.69 \\
\quad 0.01 & 86.52 \\
\quad 0.03 & 86.93 \\
\quad 0.05 & \textbf{87.04} \\
\quad 0.10 & 86.99 \\
\quad 0.20 & 86.81 \\
\multicolumn{2}{l}{\textit{Cap on $c$}}\\
\quad 2.0 & 86.90 \\
\quad $\sqrt{2\log r}$ & \textbf{87.04} \\
\quad 3.5 & 87.01 \\
\quad $\infty$ & 86.98 \\
\multicolumn{2}{l}{\textit{SURE update interval (steps)}}\\
\quad 1 & \textbf{87.06} \\
\quad 10 & 87.05 \\
\quad 20 & 87.04 \\
\quad 50 & 87.00 \\
\quad 100 & 86.91 \\
\multicolumn{2}{l}{\textit{EMA factor on $c$}}\\
\quad 0.0 & 86.95 \\
\quad 0.5 & 87.01 \\
\quad 0.9 & \textbf{87.04} \\
\quad 0.99 & 86.97 \\
\multicolumn{2}{l}{\textit{Explaining-away step}}\\
\quad $\eta{=}1/\lambda_{\max}$ (default) & 87.04 \\
\quad $\eta{=}1/2$ fixed & 86.99 \\
\quad $\eta{=}1$ fixed & 86.62 \\
\quad two EA steps & \textbf{87.07} \\
\quad three EA steps & 87.06 \\
\multicolumn{2}{l}{\textit{Noise-scale estimator}}\\
\quad MAD (default) & 87.04 \\
\quad RMS & 86.64 \\
\quad per-layer EMA & 86.58 \\
\quad iterated MAD & \textbf{87.06} \\
\quad $\|x\|_2/\sqrt{d}$ & 86.21 \\
\multicolumn{2}{l}{\textit{Scaling $\alpha$ ($r{=}32$)}}\\
\quad 16 & 86.71 \\
\quad 32 & 86.94 \\
\quad 64 & \textbf{87.04} \\
\quad 128 & 86.88 \\
\bottomrule
\end{tabular}
\caption{Sensitivity on LLaMA-3.1-8B (commonsense average; bold: best in each block). Defaults: warm-up $0.05$, cap $\sqrt{2\log r}$, interval $20$, EMA $0.9$, $\eta=1/\lambda_{\max}$, MAD, $\alpha=64$.}
\label{tab:sens}
\end{table}

\paragraph{Cost.}
Table~\ref{tab:overhead} compares training time, inference latency, and memory. \method{} trains \trainOverhead\% slower than LoRA and runs \inferOverhead\% slower at inference (unmerged, batch 1), cheaper than routed mixtures and DoRA. Merged LoRA is $12\%$ faster than unmerged LoRA; no input-dependent adapter can be merged.

\begin{table}[h]
\centering
\footnotesize
\setlength{\tabcolsep}{2.6pt}
\resizebox{\columnwidth}{!}{%
\begin{tabular}{l c c c c c}
\rowcolor{tabhead}
\Th{Method} & \Th{Params (\%)} & \Th{Rank} & \Th{Train} & \Th{Latency} & \Th{Mem.\ (GB)} \\
LoRA & 1.04 & 32 & 1.00$\times$ & 1.00$\times$ & 38.2 \\
DoRA & 1.05 & 32 & 1.27$\times$ & 1.19$\times$ & 41.6 \\
HydraLoRA & 1.05 & 16 & 1.08$\times$ & 1.06$\times$ & 38.6 \\
MixLoRA & 1.07 & 8 & 1.16$\times$ & 1.14$\times$ & 39.4 \\
LD-MoLE & 1.07 & 10.8 & 1.19$\times$ & 1.17$\times$ & 39.7 \\
SMoRA & 1.53 & 8 & 1.13$\times$ & 1.12$\times$ & 39.9 \\
FlyLoRA & 0.55 & 8 & 1.02$\times$ & 1.03$\times$ & 36.1 \\
\ourname{\method{}} & 1.04 & 8.6 & 1.05$\times$ & 1.05$\times$ & 38.7 \\
\tabend
\end{tabular}}
\caption{Cost on LLaMA-3.1-8B (training time and inference latency relative to unmerged LoRA; peak training memory per GPU).}
\label{tab:overhead}
\end{table}

\paragraph{Per-module statistics.}
Table~\ref{tab:modules} collects the per-module numbers behind Figure~\ref{fig:mechanism}. Modules with more coherent co-active values have larger gaps between the marginal and the joint code, and explaining away removes $56$--$70\%$ of the gap, the most in FFN modules.

\begin{table}[h]
\centering
\footnotesize
\setlength{\tabcolsep}{2.4pt}
\resizebox{\columnwidth}{!}{%
\begin{tabular}{l c c c c c c c}
\toprule
& & & \multicolumn{2}{c}{\textbf{EA (\% of actives)}} & & \multicolumn{2}{c}{\textbf{Gap to joint code}} \\
\cmidrule(lr){4-5}\cmidrule(lr){7-8}
\textbf{Module} & $\bar{\bm k}$ & $c_\ell$ & \textbf{inhib.} & \textbf{recr.} & $\bar\mu$ & \textbf{marginal} & \textbf{EA} \\
\midrule
\texttt{q} & 6.6 & 2.21 & 6.1 & 9.4 & 0.048 & 0.16 & 0.07 \\
\texttt{k} & 4.1 & 2.38 & 4.8 & 8.2 & 0.041 & 0.14 & 0.06 \\
\texttt{v} & 8.0 & 1.74 & 9.7 & 14.9 & 0.062 & 0.23 & 0.08 \\
\texttt{o} & 8.4 & 1.69 & 10.4 & 15.8 & 0.064 & 0.24 & 0.08 \\
\texttt{gate} & 11.0 & 1.47 & 15.2 & 23.1 & 0.079 & 0.33 & 0.10 \\
\texttt{up} & 12.1 & 1.39 & 16.8 & 25.6 & 0.083 & 0.36 & 0.11 \\
\texttt{down} & 10.0 & 1.52 & 13.1 & 20.4 & 0.074 & 0.30 & 0.09 \\
\bottomrule
\end{tabular}}
\caption{Per-module statistics on LLaMA-3.1-8B (commonsense adapter): mean active rank, SURE multiple, share of first-pass actives inhibited and of new components recruited by explaining away, mean $|\cos|$ of co-active values, and gap to the joint code before and after explaining away.}
\label{tab:modules}
\end{table}

\paragraph{Budgeted variant.}
Table~\ref{tab:budget} lists the numbers behind Figure~\ref{fig:results}b. At every target rank, the budgeted squelch outperforms top-$k$ rules with the same average number of active components. The margin is largest at small budgets, where choosing \emph{which} tokens receive the few components matters most.

\begin{table}[h]
\centering
\footnotesize
\setlength{\tabcolsep}{3.2pt}
\resizebox{\columnwidth}{!}{%
\begin{tabular}{c c c c c}
\toprule
\textbf{Mean rank} $\bar{\bm k}$ & \textbf{\method{} (target $\bar k$)} & \textbf{\method{} top-$k$} & \textbf{SMoRA} & \textbf{FlyLoRA} \\
\midrule
2 & 85.31 & 85.02 & 84.62 & 84.95 \\
4 & 86.18 & 85.84 & 85.33 & 85.63 \\
6 & 86.71 & --- & --- & --- \\
8 & 86.98 & 86.29 & 85.82 & 86.11 \\
12 & 87.08 & --- & --- & --- \\
16 & 86.93 & 86.12 & 85.64 & 85.97 \\
\midrule
SURE ($\bar k{=}8.6$) & \textbf{87.04} & --- & --- & --- \\
\bottomrule
\end{tabular}}
\caption{Commonsense average on LLaMA-3.1-8B at matched mean active rank. ``\method{} top-$k$'' replaces the floor by a fixed top-$k$ but keeps explaining away.}
\label{tab:budget}
\end{table}

\paragraph{Training dynamics and per-token ranks.}
Figure~\ref{fig:dynamics} shows that the mean active rank falls from $32$ to about $9$ within the first tenth of training and then stays nearly constant, while the SURE multiples settle within the first tenth of training. Figure~\ref{fig:tokens} shows the distribution of per-token ranks and examples. Numbers and quantity words (``48'', ``half'', ``many'') recruit the most components, and function words and punctuation the fewest. On an out-of-domain MMLU question, most tokens keep zero to two components.

\begin{figure*}[t]
\centering
\includegraphics[width=0.94\textwidth]{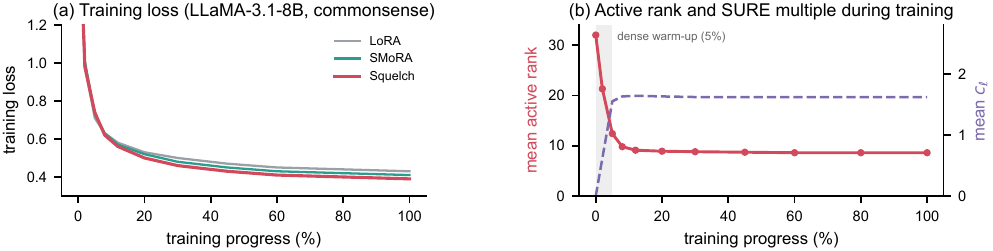}
\caption{Training dynamics on LLaMA-3.1-8B (commonsense): (a) training loss; (b) mean active rank (left) and mean SURE multiple (right).}
\label{fig:dynamics}
\end{figure*}

\begin{figure*}[t]
\centering
\includegraphics[width=0.96\textwidth]{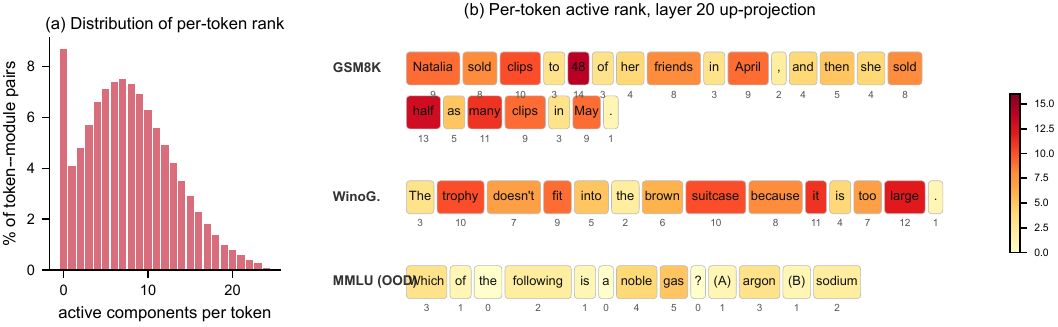}
\caption{Per-token ranks on LLaMA-3.1-8B. (a) Distribution over token--module pairs (math and commonsense evaluation data). (b) Active rank of each token in the layer-20 up-projection for a GSM8K problem, a WinoGrande sentence, and an out-of-domain MMLU question (commonsense adapter).}
\label{fig:tokens}
\end{figure*}

\section{Questions and Answers}
\label{app:faq}

\paragraph{Is the squelch just a ReLU gate on LoRA ranks?}
No. A ReLU gate, or SoRA's proximal gate~\citep{ding2023sora}, thresholds at a fixed or learned level that is shared by all inputs, and its sparsity is set by a penalty weight. The squelch thresholds each token relative to that token's own noise level, which makes the active set invariant to the token's norm (Proposition~\ref{prop:lora}(ii)), and it sets the multiple by an unbiased risk estimate, without a penalty weight or a target sparsity. Table~\ref{tab:ablation} shows that both choices matter: a per-layer scale loses $0.46$ and a threshold learned by the task loss collapses toward a dense code.

\paragraph{Why not learn a router, as mixture methods do?}
A router adds parameters, must be balanced, and learns from a loss that rewards activating more components. \method{} routes through the keys that the adapter already has and decides the number of active components with a label-free statistic~\citep{peng2024energy}. At matched average rank it outperforms routed rank-wise mixtures (Table~\ref{tab:budget}).

\paragraph{Are the responses really Gaussian?}
Only approximately. The median estimate needs only a symmetric bulk and at most half of the components carrying signal, and SURE needs only an approximately Gaussian bulk near the threshold. Heavier-tailed noise~\citep{kunstner2024heavy} would make the universal cap too permissive, but SURE's multiples stay well below the cap (Figure~\ref{fig:mechanism}c), and the sensitivity to the cap is small (Table~\ref{tab:sens}).

\paragraph{What happens when a token needs many components?}
The median estimate then errs upward (Appendix~\ref{app:more-theory}), which raises the threshold, so the squelch is conservative for such tokens. Explaining away partly compensates by recruiting complementary components (Figure~\ref{fig:mechanism}d). A larger dictionary helps more than a lower threshold (Figure~\ref{fig:results}a).

\paragraph{Does \method{} save compute?}
Not much at training time. The down-projection must be computed for all components to evaluate the responses, and LoRA's cost is small to begin with. Its purpose is accuracy and selectivity at LoRA's parameter count. At inference, inactive components skip the up-projection, but merged LoRA remains faster~\citep{sheng2024slora} (Table~\ref{tab:overhead}).

\paragraph{Is the comparison fair to the mixtures?}
All methods have the same target modules, data, optimizer, schedule, and evaluation, and all but FlyLoRA and SMoRA are matched to LoRA's parameter count. FlyLoRA has fewer trainable parameters by design, and SMoRA has more because of its router. Learning rates were tuned per method (Appendix~\ref{app:setup}).

\paragraph{Can \method{} be combined with other LoRA improvements?}
Yes. The squelch only changes how components are combined per token, so it composes with initialization schemes, magnitude--direction decompositions, and optimizer corrections~\citep{meng2024pissa,liu2024dora,wang2024loraga,hayou2024loraplus}. We have not tested these combinations.

\section{Discussion}
\label{app:discussion}

\paragraph{Why a noise model instead of a router.}
A router must learn from the task loss both where to send a token and how many components to use, and the task loss rewards using more. \method{} separates the two questions. Where to send a token is answered by the keys, as in other self-routing designs~\citep{zou2026flylora,lu2026moram}. How many components it needs is answered by comparing each response with what an irrelevant component would produce. The second question has a label-free, statistically grounded answer, and the rest of the method follows from it: a robust scale, an unbiased risk, and a joint code.

\paragraph{What the noise floor is.}
The floor is the response of components whose keys are unrelated to the token. Because LoRA's down-projection moves little from its random initialization~\citep{zhu2024asymmetry}, most keys remain close to random directions, and their responses are close to Gaussian with a scale proportional to $\|x\|$; this is why the bulk in Figure~\ref{fig:motivation}a matches an untrained adapter. As training proceeds, keys that align with a token's directions leave the bulk, and the squelch lets exactly those through. The MAD scale makes the rule invariant to the token's norm (Proposition~\ref{prop:lora}(ii)), which matters because norms vary widely across tokens and layers~\citep{sun2024massive}.

\paragraph{What explaining away does and does not do.}
Explaining away turns marginal decisions into a joint one: it removes double counting among overlapping values and recruits components that are needed only together. It does \emph{not} make the dictionary itself less redundant, and our synthetic study shows that near-duplicate components can stay aligned when it is used (Appendix~\ref{app:synthetic}). Methods that disentangle co-activated components through a regularizer target a different quantity, the coherence of the dictionary. In our experiments, such a penalty partially substitutes for explaining away but adds nothing on top of it (Table~\ref{tab:ablation}), and helps only when the dictionary is too small for the task (Table~\ref{tab:syn-capacity}).

\paragraph{When \method{} helps less.}
Tasks whose adaptation is close to a uniform shift, such as MMLU in our experiments, gain little from input-dependent rank, because most tokens need the same few components. Very small dictionaries ($r\le8$) leave little room for selection. Settings that require merging the adapter into the backbone for zero-overhead inference should use LoRA, or distill the squelched adapter into a merged one~\citep{bensaid2026slad}.

\paragraph{Extensions.}
The same statistical view applies to other ways of combining many small components per input: routing among full experts, where the responses are router logits; sparse autoencoders, where SURE could set per-feature thresholds without a sparsity target; and continual learning, where tokens that no old component answers could trigger new components. Applying the squelch in reinforcement-learning fine-tuning, where gradient-based allocation has been reported to fail~\citep{sawant2026gradient}, is a natural next test.

\paragraph{Broader impact.}
\method{} makes fine-tuning more parameter- and compute-efficient and leaves out-of-domain inputs closer to the base model. The first lowers the cost of adaptation for everyone, including malicious users; the second reduces, but does not remove, the risk that fine-tuning degrades safety behavior~\citep{qi2024safety}.

\end{document}